%% file: main-experts.tex
\documentclass[11pt]{article}
\usepackage[margin=1in]{geometry}

\usepackage{etoolbox}
\usepackage{macros}

\newtoggle{arxiv}

\toggletrue{arxiv}

\bluehyperref

\usepackage[style=alphabetic,backend=bibtex,maxnames=12,maxbibnames=10,maxcitenames=10,maxalphanames=10,giveninits=true,doi=false,url=true]{biblatex}
\newcommand*{\citet}[1]{\AtNextCite{\AtEachCitekey{\defcounter{maxnames}{2}}} \textcite{#1}}

\newcommand*{\citep}[1]{\cite{#1}}

\bibliography{bib}
\let\citealp\citep

\usepackage{hyperref}
\usepackage[capitalize]{cleveref}
\usepackage{crossreftools}
\usepackage{tikz}
\usepackage{pgfplots}

\usepackage[noend]{algorithmic}
\usepackage{algorithm}

\usepackage{subcaption}

\usepackage{comment}

\usepackage{cellspace,                  
	tabularx}

\newcolumntype{C}{>{\centering\arraybackslash}X}
\addparagraphcolumntypes{C}
\newcolumntype{P}[1]{>{\arraybackslash}p{#1}}
\usepackage{array}
\usepackage{multirow}
\renewcommand{\arraystretch}{2}
\newcolumntype{x}[1]{%
	>{\raggedleft\hspace{0pt}}p{#1}}%
\usepackage{soul}

\newcommand{\hac}[1]{
		\textcolor{blue}{\textbf{HA:} {#1}}
}
\newcommand{\tk}[1]{
		\textcolor{magenta}{\textbf{TK:} {#1}}
}
\newcommand{\kt}[1]{
		\textcolor{red}{\textbf{KT:} {#1}}
}

\renewcommand{\hac}[1]{}
\renewcommand{\tk}[1]{}
\renewcommand{\kt}[1]{}

\title{Private Online Prediction from Experts:\\ Separations and Faster Rates}
\author{%
    Hilal Asi\thanks{Apple; \texttt{hilal.asi94@gmail.com} } 
    \and Vitaly Feldman\thanks{Apple; \texttt{vitaly.edu@gmail.com}.}
    \and Tomer Koren\thanks{Blavatnik School of Computer Science, Tel Aviv University; \texttt{tkoren@tauex.tau.ac.il}.}
    \and Kunal Talwar\thanks{Apple; \texttt{kunal@kunaltalwar.org}.}
    }

\begin{document}

\maketitle

\input{abstract}

\input{introduction}

\input{pre}

\input{upper-bounds}

\input{lower-bounds}


\input{oco-imp}

\subsection*{Acknowledgements}

This work has received support from the Israeli Science Foundation (ISF) grant no.~2549/19 and the Len Blavatnik and the Blavatnik Family foundation.

\printbibliography

\appendix

\input{appendix-ub}

\input{appendix-LB}

\input{appendix-oco-imp}

\input{appendix-chernoff}

\end{document}

%% file: abstract.tex
\begin{abstract}
    Online prediction from experts is a fundamental problem in machine learning and several works have studied this problem under privacy constraints. We propose and analyze new algorithms for this problem that improve over the regret bounds of the best existing algorithms for non-adaptive adversaries. For approximate differential privacy, our algorithms achieve regret bounds of $\wt O(\sqrt{T \log d} + \log d/\eps)$ for the stochastic setting and $\wt O(\sqrt{T \log d} + T^{1/3} \log d/\eps)$ for oblivious adversaries (where $d$ is the number of experts). For pure DP, our algorithms are the first to obtain sub-linear regret for oblivious adversaries in the high-dimensional regime $d \ge T$. Moreover, we prove new lower bounds for adaptive adversaries. Our results imply that unlike the non-private setting, there is a strong separation between the optimal regret for adaptive and non-adaptive adversaries for this problem. Our lower bounds also show a separation between pure and approximate differential privacy for adaptive adversaries where the latter is necessary to achieve the non-private $O(\sqrt{T})$ regret. 
    
\end{abstract}

%% file: introduction.tex
\section{Introduction}
\label{sec:intro}

We study the problem of differentially private online prediction from experts (DP-OPE), where the algorithm interacts with an adversary for $T$ rounds. In each round, the algorithm picks an expert $x_t \in [d]$ and the adversary chooses a loss function $\ell_t: [d] \to [0,1]$. The algorithm incurs loss $\ell_t(x_t)$ at round $t$, and the goal is to design algorithms that minimize the regret, that is, the cumulative loss compared to the best fixed expert in hindsight, defined as
\begin{equation*}
    Reg_T = \sum_{t=1}^T \ell_t(x_t) -  \min_{x^\star \in [d]} \sum_{t=1}^T \ell_t(x^\star).
\end{equation*}

Online prediction from experts is an important problem in machine learning with numerous applications~\citep{AroraHaKa12}.
Without any privacy restrictions, the power of the adversary (oblivious adversary that picks the losses in advance versus adaptive adversary that picks the losses online in response to the algorithm) does not change the optimal rates for this problem~\citep{cesa2006prediction}. This has perhaps led prior work in private online experts to focus on the strongest notion of adaptive adversaries~\citep{JainKoTh12,SmithTh13,JainTh14,AgarwalSi17}. In this work, we study the problem of DP-OPE against oblivious adversaries as well and show that, somewhat surprisingly, the power of the adversary can significantly affect the optimal rates for this problem.

We consider three types of adversaries: the strongest, \emph{adaptive adversary}, can choose the loss sequence $(\ell_t)_{t=1}^T$ adversarially in an online manner, where the loss $\ell_t$ may depend arbitrarily on the choices made by the algorithm in previous time steps $(x_\ell)_{\ell=1}^{t-1}$. A slightly weaker notion is that of an \emph{oblivious adversary}, which chooses a sequence of loss functions in advance. The weakest adversary is the \emph{stochastic adversary} which chooses a distribution over loss functions and at each round samples a loss function i.i.d.~from this distribution.

In the classical non-private setting, all of these adversarial models are equivalent, in the sense that they all induce an optimal rate of order $\sqrt{T \log d}$~\citep{cesa2006prediction}. We study algorithms that are required to be differentially private, where we view the sequence of loss functions as the dataset and adjacent datasets differ in a single loss function. We note that all our algorithms are private with respect to adaptive adversaries (see~\cref{sec:pre} for precise definitions) and only the utility bounds assume a weaker adversary. 

Under the constraint of \ed-differential privacy, the best existing results obtain regret of order
$$
\sqrt{T \log d} + \min\left( \frac{1}{\diffp} \sqrt{T \log(1/\delta) \log d}, \frac{1}{\diffp} \sqrt{d \log(1/\delta)} \log d \log^2 T \right) 
$$ 
for adaptive adversaries~\citep{JainTh14,AgarwalSi17}. For pure $\diffp$-differential privacy, the best existing regret~\citep{JainTh14} is of order 
$$\sqrt{T \log d} +  \frac{d  \log d \log^2 T}{\diffp}.$$ 
However, none of existing results prove any lower bounds (besides the trivial non-private lower bound $\sqrt{T \log d}$) to certify the optimality of these rates; thus, it is currently unclear whether these rates are optimal for adaptive adversaries, let alone for oblivious and stochastic adversaries.



\subsection{Our contributions}
We study DP-OPE for different types of adversaries and develop new algorithms and lower bounds. More precisely, we obtain the following results.

\paragraph{Faster rates for oblivious adversaries (\cref{sec:ub-obl-sd}):} 
We develop new algorithms for DP-OPE with oblivious adversaries based on a lazy version of the multiplicative weights algorithm. For pure $\diffp$-DP, our algorithms obtain regret $$\frac{\sqrt{T} \log d}{\diffp}.$$ 
This is the first algorithm to achieve sub-linear regret for pure DP in the high-dimensional regime $d \ge T$. For approximate \ed-DP, our algorithm achieves regret $$ \sqrt{T \log d } + \frac{T^{1/3} \log^{1/3}(1/\delta) \log d}{\diffp}.$$ 
This essentially demonstrates that the privacy cost for DP-OPE is negligible as long as $\diffp \gg T^{-1/6}$. In contrast, previous work has established privacy cost $\diffp^{-1} \sqrt{T \log d\log(1/\delta)}$ which is larger than the non-private cost even when $\diffp$ is constant and $\delta=1/T$. See~\cref{fig:table-appr,fig:table-pure} for more details.

\paragraph{Separation between adaptive and oblivious adversaries (\cref{sec:lower-bounds}):} 
We prove the first lower bounds for DP-OPE with adaptive adversaries that are stronger than the non-private lower bounds. These bounds show that any private algorithm must suffer linear regret if $\diffp \le 1/\sqrt{T}$ for approximate DP and $\diffp \le 1/10$ for pure DP. As our algorithms for oblivious adversaries obtain sub-linear regret in this regime of privacy parameters, this demonstrates that the oblivious model is significantly weaker than the adaptive model in the private setting (see~\cref{fig:comp}). Moreover, these lower bounds show a separation between pure and approximate DP for DP-OPE with adaptive adversaries as the latter is necessary to obtain sub-linear regret.

\paragraph{Near-optimal rates for stochastic adversaries (\cref{sec:ub-stoch}):}  
We design a general reduction that transforms any algorithm for private stochastic optimization in the offline setting into an algorithm for private online optimization with nearly the same rates (up to logarithmic factors). By building on algorithms for the offline setting~\citep{AsiFeKoTa21}, we obtain regret $O(\sqrt{T\log d} +  \diffp^{-1}\log d \log T)$ for DP-OPE with stochastic adversaries. Moreover, using this reduction with general algorithms for differentially private stochastic convex optimization (DP-SCO)~\citep{FeldmanKoTa20}, we obtain near-optimal regret $O( \sqrt{T} + \diffp^{-1} \sqrt{d} \log T )$ for the problem of DP-OCO (online convex optimization) with stochastic adversaries, improving over the best existing result $\sqrt{T}d^{1/4}/\diffp $~\citep{KairouzMcSoShThXu21}.


\paragraph{Improved rates for DP-OCO (\cref{sec:oco-imp}):}
Building on our improvements for DP-OPE, we improve the existing rates for DP-OCO where the algorithms picks $x_t \in \mc{X} = \{ x \in \R^d: \ltwo{x} \le 1\}$ and the adversary picks $\ell_t : \mc{X} \to \R$. Our rates improve over the rates of~\cite{KairouzMcSoShThXu21} in certain regimes.


    
    
    
    

\renewcommand{\arraystretch}{2}

\begin{table}[h]
\begin{center}
    \begin{subtable}[h]{0.8\textwidth}
		\begin{tabular}{| c | c | c |}
		    \hline
			  & \textbf{\darkblue{Prior work}} & \textbf{\darkblue{This work}}\\
			\hline
			{{\textbf{Stochastic}}} & \multirow{3}{*}{{$ \sqrt{T \log d} + \min\left( \frac{\sqrt{T  \log d}}{\diffp}, \frac{\sqrt{d } \log  d}{\diffp}    \right)$}  } & $   \sqrt{T\log d} +  \frac{\log d }{\diffp} $      \\
			\cline{1-1} \cline{3-3}
			\textbf{{Oblivious}} & & $\sqrt{T \log d } + \frac{T^{1/3}  \log d}{\diffp}$ \\ 
			\cline{1-1} \cline{3-3}
			{\textbf{Adaptive}} & & None \\ 
			\hline
		\end{tabular}
    \caption{Approximate \ed-DP.}
    \label{fig:table-appr}
\end{subtable}

\vspace{.5cm}
\begin{subtable}[h]{0.8\textwidth}
        \hspace{2cm}
		\begin{tabular}{| c | c | c |}
		    \hline
			  & \textbf{\darkblue{Prior work}} & \textbf{\darkblue{This work}}\\
			\hline
			{\textbf{Stochastic}} & \multirow{3}{*}{{$ \sqrt{T \log d} +  \frac{d \log d}{\diffp}$}  } & $   \sqrt{T\log d} +  \frac{\log d }{\diffp} $      \\
			\cline{1-1} \cline{3-3}
			{\textbf{Oblivious}} & & $\frac{\sqrt{T} \log d}{\diffp}$ \\ 
			\cline{1-1} \cline{3-3}
			{\textbf{Adaptive}} & & None \\ 
			\hline
		\end{tabular}
    \caption{Pure $\diffp$-DP.}
    \label{fig:table-pure}
\end{subtable}
     \caption{Upper bounds for DP-OPE with different types of adversaries. For readability, we omit logarithmic factors that depend on $T$ and $1/\delta$.}
     \label{tab:temps}
     \end{center}
\end{table}

\begin{figure}[htb] 
  \begin{subfigure}{0.5\textwidth}
    \centering
    \begin{tikzpicture}
\begin{axis}[
width=7.0cm,               
grid=major,
xmax=1,xmin=0,
ymin=0.4,ymax=1.1,
xlabel={$\alpha$},ylabel={$\beta$},
legend style={at={(0.7,0.46)},anchor=north},
tick label style={font=\small}
]

\addplot+[color=red,domain=0:1,thick,mark=none] {min(max(1/3 + x,1/2),1)};
\addlegendentry{\tiny DP-SD (Cor.~\ref{cor:sd-appr})}

\addplot+[color=blue,domain=1/3:1,thick,mark=none] {min(2/5 + 4*x/5,1)};
\addlegendentry{\tiny Batch DP-SD (Cor.~\ref{cor:sd-batch})}

\addplot+[color=black,domain=0:1,thick,mark=none] {min(max(1/2 + x,1/2),1)};
\addlegendentry{\tiny Prior work}

\addplot+[color=black,domain=0:1/2,thick,dashed,mark=none] {max(1/2,2*x)};
\addlegendentry{\tiny LB (adaptive adversary)}



\end{axis}
\end{tikzpicture}
    \caption{Approximate \ed-DP}
  \end{subfigure}%
  \begin{subfigure}{0.5\textwidth}
    \centering
    \begin{tikzpicture}
\begin{axis}[
width=7.0cm,               
grid=major,
xmax=1,xmin=0,
ymin=0.4,ymax=1.1,
xlabel={$\alpha$},
legend style={at={(0.7,0.4)},anchor=north},
tick label style={font=\small}
]

\addplot+[color=red,domain=0:1,thick,mark=none] {min(max(1/2 + x,1/2),1)};
\addlegendentry{\tiny DP-SD (Cor.~\ref{cor:sd-appr})}


\addplot+[color=black,domain=0:1,thick,mark=none] {1};
\addlegendentry{\tiny Prior work}

\addplot+[color=black,domain=0:1/2,thick,dashed,mark=none] {1};
\addlegendentry{\tiny LB (adaptive adversary)}



\end{axis}
\end{tikzpicture}
    \caption{Pure $\diffp$-DP}
  \end{subfigure}
  \caption{Regret bounds for private prediction from experts against oblivious adversaries for the high-dimensional regime $d \ge T$. 
We denote the privacy parameter $\diffp = T^{-\alpha}$ and regret $T^\beta$, and plot $\beta$ as a function of $\alpha$ (ignoring logarithmic factors).  }\label{fig:comp}
\end{figure}
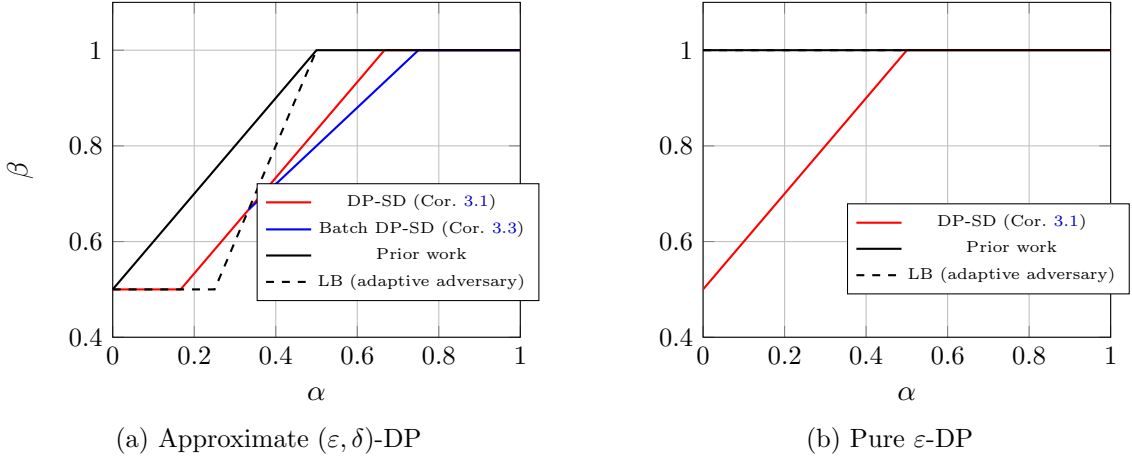

\subsection{Related work}
\citet{DworkNaPiRo10} were the first to study differentially private learning in the online setting and introduced the binary tree mechanism which is an important building block of many private algorithms in the online setting. In our context of online prediction from experts, there has been several works that study this problem under the constraint of differential privacy~\citep{JainKoTh12,SmithTh13,JainTh14,AgarwalSi17}. The best existing algorithms depend on the dimensionality regime: in the high-dimensional setting, \citet{JainTh14} developed an algorithm based on follow-the-regularized-leader (FTRL) with entropy regularization that achieves regret $O(\diffp^{-1}\sqrt{T \log d \log(1/\delta)})$ for \ed-DP. For low dimensional problems, \citet{AgarwalSi17} developed an improved algorithm that uses the binary tree mechanism to estimate the running sum of the gradients in the FTRL framework. Their algorithm achieves regret $O(\sqrt{T\log d} + \diffp^{-1} d \log d \log^2 T)$  for $\diffp$-DP. Moreover, extending their algorithm to \ed-DP results in regret  $O(\sqrt{T\log d} + \diffp^{-1} \sqrt{d \log(1/\delta)} \log d \log^2 T)$. More recently, \citet{AsiFeKoTa22b} considered DP-OPE in the realizable regime where there is a zero-loss expert, and propose new algorithms that obtain near-optimal regret $\mathsf(poly)(\log d)/\diffp$ up to logarithmic factors.

A slightly related and somewhat easier problem is that of differentially private stochastic convex optimization (DP-SCO) which has been extensively investigated recently~\citep{BassilySmTh14,BassilyFeTaTh19,FeldmanKoTa20,AsiFeKoTa21,AsiDuFaJaTa21}. In DP-SCO, we are given $n$ samples from some distribution and we wish to minimize the excess population loss. In contrast to the online setting, here all $n$ samples are given to the algorithm and it is required to produce only one model; this makes the online version harder as the algorithm has to output a model at each time-step. However, we note that our reduction in~\cref{sec:ub-stoch} shows that DP-SCO is essentially as hard as DP-OCO with stochastic adversaries (up to logarithmic factors). For oblivious and adaptive adversaries, the online setting may be harder as it allows general loss functions that are not necessarily generated from a distribution.

Perhaps most closely related to our problem in the non-private setting is online learning with limited switching~\citep{KalaiVe05,GeulenVoWi10,AltschulerTa18,chen2020minimax,ShermanKo21}. In this setting, the algorithm aims to minimize the regret when having a budget of at most $S$ switches, that is, it can update its decision at most $S$ times. This constraint is (informally) somewhat related to privacy as the less you update the model, the less information is leaked about the data. The ideas developed in this literature turn out to be useful for our DP-OPE setting. Indeed, our algorithms in~\cref{sec:ub-obl-sd} build on ideas from~\cite{GeulenVoWi10} which developed a lazy version of multiplicative weights to limit the number of switches. Moreover,
similarly to our results, the hardness of online learning problems with limited switching can depend on the power of the adversary. For example, for OCO with limited switching, the optimal rate is $\sqrt{T} + T/S$ for oblivious adversaries and $\sqrt{T} + T/\sqrt{S}$ for adaptive adversaries~\citep{chen2020minimax,ShermanKo21}. Despite of these similarities, our results do not establish a fundamental connection between privacy and switching constraints and we leave this as an open question for future research.

%% file: pre.tex
\newcommand{\Adv}{\mathsf{Adv}}
\newcommand{\Fl}{\mc{L}} 

\section{Problem setting and preliminaries}
\label{sec:pre}
{Online prediction from experts} (OPE) is an interactive $T$-round game between an online algorithm $\A$ and adversary $\Adv$. At round $t$, the algorithm $\A$ chooses an expert $x_t \in [d]$ and the adversary $\Adv$ picks a loss function $\ell_t \in \Fl = \{\ell \mid  \ell : [d] \to [0,1]\}$ simultaneously. We let $\A_t(\ell_1,\dots,\ell_{t-1}) = x_t$ denote the mapping of algorithm $\A$ at round $t$. Similarly, we define $\Adv_t$ to be the mapping of the adversary at round $t$ (we provide more details on the input of $\Adv_t$ below depending on the type of the adversary).
The algorithm observes $\ell_t$ (after choosing $x_t$) and incurs loss $\ell_t(x_t)$. For a predefined number of rounds $T$, the regret of the algorithm $\A$ is defined as 
\begin{equation*}
    \reg_T(\A) = \sum_{t=1}^T \ell_t(x_t) - \min_{x^\star \in [d]} \sum_{t=1}^T \ell_t(x^\star).
\end{equation*}
%
%
%
We consider three types of adversary $\Adv$ for choosing the sequence of loss functions $\{\ell_t\}_{t=1}^T$. To help define privacy, we will diverge from the traditional presentation of these models in online learning literature. The adversary will consist of a sequence of $T$ data points $z_1,\dots,z_T \in \domain$, and an algorithm $\ell$ that generates the sequence of losses.

For both stochastic and oblivious adversaries the loss function at step $t$ is generated based on data point $z_t$ alone; i.e. $\ell_t(\cdot) = \ell(\cdot; z_t)$, where for all $z \in \domain$, $\ell(\cdot; z)$ is an admissible loss function for the relevant setup. The two models differ in the choice of the sequence $z_1,\ldots,z_T$: for a stochastic adversary, the sequence of $z_i$'s is chosen i.i.d. from some distribution $P$ (chosen by the adversary). For an oblivious adversary, this sequence itself is adversarially chosen. In other words:
\begin{align*}
    Reg_T^{\mbox{(stochastic)}}(\A) &= \sup_{\ell, P} \mathbb{E}_{z_1,\ldots,z_T \sim P^T} [Reg_T(\A) | \ell_t(\cdot) = \ell(\cdot; z_t)],\\
    Reg_T^{\mbox{(oblivious)}}(\A) &= \sup_{\ell} \sup_{z_1,\ldots,z_T \in \domain^T} [Reg_T(\A) | \ell_t(\cdot) = \ell(\cdot; z_t)].
\end{align*}

In the case of an adaptive adversary, the loss at step $t$ can depend on the algorithm's choices in previous steps. Thus $\ell_t(\cdot) = \ell(\cdot; z_t, x_{1:t-1})$, where as before this loss function is constrained to be an admissible loss function for all possible values of inputs $z_t, x_{1:t-1}$. The adaptive regret is then the worst case regret over $z_1,\ldots,Z_T$ and the mapping $\ell$:
\begin{align*}
Reg_T^{\mbox{(adaptive)}}(\A) &= \sup_{\ell} \sup_{z_1,\ldots,z_T \in \domain^T} [Reg_T(\A) | \ell_t(\cdot) = \ell(\cdot; z_t, x_{1:t-1})].
\end{align*}


Given data $z_{1:T} = (z_1,\dots,z_T) \in \domain^T$, we let $\A \circ \Adv(z_{1:T}) = x_{1:T}$ denote the output of the interaction between the algorithm $\A$ and adversary $\Adv$ given inputs $z_{1:T}$.

Under this setting, the goal is to design private algorithms that minimize the appropriate notion of regret. To this end, we extend the standard definition of \ed-differential privacy~\citep{DworkMcNiSm06,DworkKeMcMiNa06} to the online setting. 
Like most previous works, we study a stronger notion of privacy that holds even against adaptive adversaries.\footnote{\citet{JainRaSiSm21} recently formalized DP against adaptive adversaries for a different online learning problem. Their notion is equivalent to ours, but our presentation may be easier to work with.} 
%
%
\begin{definition}[Adaptive DP]
\label{def:DP}
	A randomized  algorithm $\A$ is \ed-differentially private against adaptive adversaries (\ed-DP) if, for all sequences $\Ds=(z_1,\dots,z_T) \in \domain^T$ and $\Ds'=(z'_1,\dots,z'_T) \in \domain^T$ that differ in a single element, for any $\ell$ defining an adaptive adversary $\Adv$, and for all events $\cO$ in the output space of $\A \circ \Adv$, we have
	\[
	\Pr[\A \circ \Adv(\Ds)\in \cO] \leq e^{\eps} \Pr[\A \circ \Adv(\Ds')\in \cO] +\delta.
	\]
\end{definition}

As remarked earlier, all our algorithms will be differentially private against adaptive adversaries, and the other adversary models are considered only from the the point of view of utility. This is consistent with a long line of work on private learning algorithms, where privacy is proven for worst-case inputs while utility bounds often make distributional or other assumptions on the data.

%% file: upper-bounds.tex
\section{Private algorithms for prediction from experts}
\label{sec:upper-bounds}
We begin our algorithmic contribution by developing new algorithms for oblivious (\cref{sec:ub-obl-sd}) and stochastic adversaries (\cref{sec:ub-stoch}). 
The main idea is to save the privacy budget by limiting the number of model updates. Though, the way in which this is done varies significantly depending on the adversary: for stochastic adversaries, we allow for $\log(T)$ updates at fixed time-steps $t=2^i$, while for oblivious adversaries we employ a more adaptive time-step strategy for updates based on the underlying data.

\subsection{Oblivious adversaries using shrinking dartboard}
\label{sec:ub-obl-sd}
In this section, we present our main algorithms for DP-OPE with oblivious adversaries. We build on the shrinking dartboard algorithm~\citep{GeulenVoWi10} to develop a private algorithm that improves over the best existing results for both pure and approximate DP. For $\diffp$-DP, our algorithm obtains regret $\diffp^{-1} \sqrt{T} \log d$ which nearly matches the non-private regret for constant $\diffp$; this is the first algorithm for pure DP that achieves sub-linear regret for oblivious adversaries. For approximate DP, our algorithm obtains regret $\sqrt{T \log d } + \diffp^{-1} (T\log(1/\delta))^{1/3} \log d$: our algorithm achieves negligible privacy cost compared to the non-private regret in the high-dimensional regime. Previous algorithms  obtain regret roughly $\diffp^{-1} \sqrt{T \log d \log(1/\delta)}$ for $d \ge T$ which is $\diffp^{-1} \sqrt{\log(1/\delta)}$ larger than the non-private regret.  

Key to our improvements in this section (similarly to limited switching OPE; \citealp{GeulenVoWi10}) is the observation that the distribution of the multiplicative weights (MW) algorithm does not change significantly between consecutive time-steps. Thus, by using a correlated sampling procedure, we can preserve the same marginal distribution as MW while updating the model with small probability. As this limits the number of model updates, this will allow us to assign higher privacy budget for each model update, hence improving its utility.

Another obstacle that arises from this approach is that the switching probability (to preserve the same marginal distribution) depends on the underlying data $p_{\mathsf{switch}} = 1 - (1-\eta)^{\ell_{t-1}(x_{t-1})}$. This is clearly not private as the probability of switching is zero whenever $\ell_{t-1}(x_{t-1})=0$. To tackle this issue, we guarantee privacy by adding another switching step that forces the algorithm to switch with probability $p$ regardless of the underlying data.

\begin{algorithm}
	\caption{Private Shrinking Dartboard}
	\label{alg:SD}
	\begin{algorithmic}[1]
	    \REQUIRE Step size $\eta > 0$, switching probability $p \in [0,1]$, switching budget $K$
		\STATE Set $w^1_i = 1$, $p^1_i = \frac{1}{d}$ for all $i \in [d]$
		\STATE Choose expert $x_1$ from the distribution $P^1 = (p^1_1,\dots,p^1_d)$
		\STATE Set $k = 1$
        \FOR{$t=2$ to $T$\,}
            \STATE Set $w^t_i = w^{t-1}_i (1 - \eta)^{\ell_{t-1}(i)}$ for all $i \in [d]$
            \STATE Set $p^t_i = \frac{w^t_i}{\sum_{i'=1}^d w^t_{i'}}$ for all $i \in [d]$
            \STATE Sample $z_t \sim \mathsf{Ber}(1-p)$
            \IF{$z_t=1$}
            \STATE Sample $z_t \sim \mathsf{Ber}(w^t_{x_{t-1}}/w^{t-1}_{x_{t-1}})$
            \ENDIF
            \IF{$z_t = 1$}
                \STATE $x_t = x_{t-1}$
            \ELSIF{$k < K$}
                \STATE $k = k + 1$
                \STATE Sample $x_t$ from $P^t$
            \ENDIF
            \STATE Receive $\ell_t: [d] \to [0,1]$
            \STATE Pay cost $\ell_t(x_t)$
        \ENDFOR
	\end{algorithmic}
\end{algorithm}

The following theorem states the privacy guarantees and the upper bounds on the regret of~\cref{alg:SD}. We defer the proof to~\cref{sec:proof-thm-ub-DS}.
\begin{theorem}
\label{thm:ub-priv-DS}
    Let $\ell_1,\dots,\ell_T \in [0,1]^d$ be chosen by an oblivious adversary. \cref{alg:SD} with $p < 1/2$, $\eta<1/2$, and $K = 4 T p $ has regret
    \begin{equation*}
     \E\left[ \sum_{t=1}^T \ell_t(x_t) - \min_{x \in [d]} \sum_{t=1}^T \ell_t(x) \right]
        \le \eta T
        + \frac{\ln d}{\eta} + 2 T e^{-Tp/3}.
    \end{equation*}
    Moreover, \cref{alg:SD} is $(\diffp,\delta)$-DP where 
    \begin{equation*}
     \diffp = 
     \begin{cases}
         \frac{5\eta}{p} + 100 T p  \eta^2 + 20 \eta \sqrt{ T p \log(1/\delta)} & \text{if } \delta>0 ;\\
         \frac{\eta}{p} + 16 T p \eta & \text{if } \delta = 0 .
     \end{cases}
    \end{equation*}
\end{theorem}

\iftoggle{arxiv}{
Before proving~\cref{thm:ub-priv-DS}, we illustrate the implication of this result when optimizing the parameters in the algorithm. We have the following regret for approximate DP (see~\cref{sec:apdx-cor-sd-appr} for the proof).}
{
We illustrate the implication of this result when optimizing the parameters in the algorithm. We have the following regret for approximate DP (see~\cref{sec:apdx-cor-sd-appr} for the proof).
}
\begin{corollary}
\label{cor:sd-appr}
    Let $\ell_1,\dots,\ell_T \in [0,1]^d$ be chosen by an oblivious adversary.
    Let $\diffp \le 1$ and set $\diffp_0 = \min(\diffp/2, \log^{1/3}(1/\delta) T^{-1/6} \sqrt{\ln d})$ and $0 < \delta \le 1$ such that $T \ge \Omega(\log(1/\delta))$.  Setting $p = 1/(T \log(1/\delta))^{1/3}$ and $\eta = p \diffp_0/20$, \cref{alg:SD} is \ed-DP and has regret 
    \begin{equation*}
     \E\left[ \sum_{t=1}^T \ell_t(x_t) - \min_{x \in [d]} \sum_{t=1}^T \ell_t(x) \right]
     \le O \left( \sqrt{T \ln d } + \frac{T^{1/3} \log^{1/3}(1/\delta) \ln d}{\diffp}  \right).
    \end{equation*}
\end{corollary}
Moreover, we have the following regret for pure DP (proof in~\cref{sec:apdx-cor-pure}).
\begin{corollary}
\label{cor:pure}
    Let $\ell_1,\dots,\ell_T \in [0,1]^d$ be chosen by an oblivious adversary.
    Let $\diffp \le 1$. Setting $p = 1/\sqrt{T}$ and $\eta = p \diffp/20$, \cref{alg:SD} is $\diffp$-DP and has regret 
    \begin{equation*}
     \E\left[ \sum_{t=1}^T \ell_t(x_t) - \min_{x \in [d]} \sum_{t=1}^T \ell_t(x) \right]
     \le O \left( \frac{\sqrt{T} \ln d}{\diffp}  \right).
    \end{equation*}
\end{corollary}

\subsubsection{Private batched shrinking dartboard}
For smaller values of $\diffp$, we present a batch version of~\cref{alg:SD} that groups losses in a batch of size $B$ and applies the same update. 
For a batch size $B$, we define the grouped loss
\begin{equation*}
    \tilde \ell_t = \frac{1}{B} \sum_{i=Bt}^{B(t+1)} \ell_i.
\end{equation*}
The batch version of~\cref{alg:SD} then runs~\cref{alg:SD} on the grouped loss $\tilde \ell_t $ for $\tilde T = \ceil{T/B}$ iterations.
\iftoggle{arxiv}{%
The following theorem state the regret of this algorithm. We prove the theorem in~\cref{sec:apdx-thm-ub-priv-DS-batch}
\begin{theorem}
\label{thm:ub-priv-DS-batch}
    Let $\ell_1,\dots,\ell_T \in [0,1]^d$ be chosen by an oblivious adversary. \cref{alg:SD} with batch size $1 \le B \le T$, $p < 1/2$, $\eta<1/2$, and $K = 4 T p/B $ has regret
    \begin{equation*}
     \E\left[ \sum_{t=1}^T \ell_t(x_t) - \min_{x \in [d]} \sum_{t=1}^T \ell_t(x) \right]
        \le \eta T
        + \frac{B \ln d}{\eta} + 2 T e^{-Tp/3B}.
    \end{equation*}
    Moreover, for $\delta>0$, \cref{alg:SD} is $(\diffp,\delta)$-DP where 
    \begin{equation*}
     \diffp = 
          \frac{5\eta}{Bp} + 100 T p  \eta^2/B^3 + \frac{20\eta}{B} \sqrt{12 T p/B \log(1/\delta)}
          . 
    \end{equation*}
\end{theorem}
Optimizing the batch size $B$, we obtain the following regret. \cref{fig:comp} illustrates that this algorithm offers improvements over the original version (without batches) in the high-privacy regime. We defer the proof to~\cref{sec:adpx-cor-sd-batch}.
\begin{corollary}
\label{cor:sd-batch}
    Let $\frac{\log^{2/3}(1/\delta) \log(d)}{T} \le \diffp \le \frac{\log^{2/3}(1/\delta) \log(d)}{T^{1/3}} $ and $\delta \le 1$. Setting $B = \frac{\log^{2/5}(1/\delta) \log^{3/5}(d)}{T^{1/5} \diffp^{3/5}} $, $p = (\frac{B}{T \log(1/\delta)})^{1/3}$ and $\eta = B p \diffp/40$, \cref{alg:SD} is \ed-DP and has regret 
    \begin{equation*}
     \E\left[ \sum_{t=1}^T \ell_t(x_t) - \min_{x \in [d]} \sum_{t=1}^T \ell_t(x) \right]
     \le O \left( \frac{T^{2/5} \log^{1/5}(1/\delta) \log^{4/5}(d))  }{\diffp^{4/5}}  + 2 T e^{-Tp/3B} \right).
    \end{equation*}
\end{corollary}}
{
Following similar steps as in the previous section, we obtain the following regret bounds for this algorithm (proof in~\cref{sec:adpx-cor-sd-batch}). \cref{fig:comp} illustrates that this algorithm offers improvements over the original version (without batches) in the high-privacy regime.
\begin{theorem}
\label{cor:sd-batch}
    Let $\frac{\log^{2/3}(1/\delta) \log(d)}{T} \le \diffp \le \frac{\log^{2/3}(1/\delta) \log(d)}{T^{1/3}} $ and $\delta \le 1$. Setting $B = \frac{\log^{2/5}(1/\delta) \log^{3/5}(d)}{T^{1/5} \diffp^{3/5}} $, $p = (\frac{B}{T \log(1/\delta)})^{1/3}$ and $\eta = B p \diffp/40$, \cref{alg:SD} is \ed-DP and has regret 
    \begin{equation*}
     \E\left[ \sum_{t=1}^T \ell_t(x_t) - \min_{x \in [d]} \sum_{t=1}^T \ell_t(x) \right]
     \le O \left( \frac{T^{2/5} \log^{1/5}(1/\delta) \log^{4/5}(d))  }{\diffp^{4/5}}  + 2 T e^{-Tp/3B} \right).
    \end{equation*}
\end{theorem}
}

\subsection{Stochastic adversaries}
\label{sec:ub-stoch}
In this section, we consider stochastic adversaries and present a reduction from private online learning problems---including OPE and online convex optimization (OCO) in general---to the (offline) differentially private stochastic convex optimization (DP-SCO)~\citep{BassilyFeTaTh19}. This reduction demonstrates that the (offline) DP-OCO problem is not harder than private online learning for stochastic adversaries: any algorithm for DP-SCO can be transformed into an online algorithm with stochastic adversaries with nearly the same rate (up to logarithmic factors). Using existing algorithms for DP-SCO~\citep{AsiFeKoTa21}, this reduction then results in regret $\sqrt{T\log d} +  \diffp^{-1} \log d \log T$ for private prediction from experts (\cref{cor:DP-exp-stoch}) and regret $\sqrt{T} + \diffp^{-1} \sqrt{d} \log T$ for general DP online convex optimization in $\ell_2$-geometry (\cref{cor:DP-OCO}).

Our reduction builds on algorithms for DP-SCO. In this problem, the loss functions are sampled i.i.d. from some distribution $\ell_i \simiid P$ where $\ell_i : \mc{X} \to \R$ and the goal is to minimize the population loss $L(x) = \E_{\ell \sim P}[\ell(x)]$ given $n$ samples $\ell_1,\dots,\ell_n \simiid P$. The performance of an algorithm $\A$ given $n$ samples is measured by its excess population loss, that is, $\Delta_n(\A) = \E[{L(\A(\ell_1,\dots,\ell_n)) - \inf_{x \in \mc{X}}L(x)}]$.


Given an algorithm $\A$ for DP-SCO,
we design an online algorithm that updates the model only a logarithmic number of times---during the time-steps $t=1,2,4,8,\dots,T$.  For each such time-step $t$, we run $\A$ on the past $t/2$ samples to release the next model. As the loss functions are from the same distribution, the previous model generated by $\A$ should perform well for future loss functions. 
We present the full details in~\cref{alg:stoch-adv}.

\begin{algorithm}
	\caption{Limited Updates for Online Optimization with Stochastic Adversaries}
	\label{alg:stoch-adv}
	\begin{algorithmic}[1]
	    \REQUIRE Parameter space $\mc{X}$, DP-SCO algorithm $\A$
		\STATE Set $x_0 \in \mc{X}$ 
		
        \FOR{$t=1$ to $T$\,}
            \IF{$t=2^\ell$ for some integer $\ell \ge 1$}
                \STATE Run an optimal \ed-DP-SCO algorithm $\A$ over $\mc{X}$ with samples $\ell_{t/2},\dots,\ell_{t-1}$.
                
                \STATE Let $x_t$ denote the output of the private algorithm
            \ELSE
                \STATE Let $x_t = x_{t-1}$
            \ENDIF
            \STATE Receive $\ell_t: \mc{X} \to \R$.
            \STATE Pay cost $\ell_t(x_t)$
        \ENDFOR
	\end{algorithmic}
\end{algorithm}
We have the following regret for~\cref{alg:stoch-adv}. We defer the proof to~\cref{sec:thm-ub-stoch-OCO}.
\begin{theorem}
\label{thm:ub-stoch-OCO}
    Let $\ell_1,\dots,\ell_T : \mc{X} \to \R$ be convex functions chosen by a stochastic adversary, $\ell_i \simiid P$. 
    Let $\A$ be a \ed-DP algorithm for DP-SCO. Then
    \cref{alg:stoch-adv} is \ed-DP and has regret
    \begin{equation*}
     \E\left[ \sum_{t=1}^T \ell_t(x_t) - \min_{x \in \mc{X}} \sum_{t=1}^T \ell_t(x) \right]
        \le \sum_{i=1}^{\log T} 2^i \Delta_{2^i}(\A).
    \end{equation*}
\end{theorem}

\newcommand{\Aopt}{\A_{\mathsf{\ell_2}}}
\newcommand{\Aone}{\A_{\mathsf{\ell_1}}}
Now we present the implications of this reduction for DP-OPE with stochastic adversaries. We use 
the algorithm for DP-SCO in $\ell_1$-geometry from~\cite{AsiFeKoTa21}. This algorithm, denoted $\Aone$, has excess population loss $\Delta_n(\Aone) = O(\sqrt{\log(d)/n} + \log(d)/n\diffp)$. We have the following result which we prove in~\cref{sec:apdx-cor-DP-exp-stoch}.
\begin{corollary}
\label{cor:DP-exp-stoch}
    Let $\ell_1,\dots,\ell_T : [d] \to [0,1]$ be  chosen by a stochastic adversary, $\ell_i \simiid P$. 
    Then
    \cref{alg:stoch-adv} using $\Aone$ is $\diffp$-DP and has regret
    \begin{equation*}
     \E\left[ \sum_{t=1}^T \ell_t(x_t) - \min_{x \in [d]} \sum_{t=1}^T \ell_t(x) \right]
        \le  O \left(\sqrt{T\log d} +  \frac{\log d \log T}{\diffp} \right).
    \end{equation*}   
\end{corollary}

%% file: lower-bounds.tex
\section{Lower bounds}
\label{sec:lower-bounds}

In this section, we prove new lower bounds for adaptive adversaries which show a separation from the non-adaptive case. In particular, we show that for $\diffp = 1/\sqrt{T}$, an adaptive adversary can force any \ed-DP algorithm to incur a linear regret $\Omega(T)$. Similarly, any $\diffp$-DP algorithm with $\diffp \le 1/10$ must incur linear regret against adaptive adversaries. 
On the other hand, the results of~\cref{sec:upper-bounds} show that a sub-linear regret is possible for both privacy regimes with oblivious adversaries. 

Our lower bound is based on finger-printing lower bound constructions~\citep{BunUV18} which are the basic technique for proving lower bounds in the offline setting. The idea is to design a reduction that uses a DP-OPE algorithm for estimating the signs of the mean of high-dimensional inputs; a problem which is known to be hard under differential privacy.  The following theorem summarizes our main lower bound for \ed-DP. We defer the proof to~\cref{sec:apdx-thm-lb-adaptive-adv}.

\begin{theorem}
\label{thm:lb-adaptive-adv}
    Let $T$ be sufficiently large and $d \ge 2 T$.
    Let $\diffp \le 1$ and $\delta \le 1/T^3$.
    If $\A$ is \ed-DP then there is an adaptive adversary such that
    \begin{equation*}
        \E\left[\sum_{t=1}^T \ell_t(x_t) - \min_{x \in [d]} \sum_{t=1}^T \ell_t(x)\right]
        \ge \Omega \left( \min \left(T, \frac{1}{(\diffp \log T)^2} \right) \right).
    \end{equation*}
\end{theorem}

\begin{proof}(sketch)
    We illustrate the main ideas of the lower bound for $\diffp \le 1/(\sqrt{T} \log T)$, where we have to prove that the regret has to be linear in this case.
    We will reduce the problem of private sign estimation to DP-OPE with adaptive adversaries and use existing lower bounds for private sign estimation (see~\cref{thm:lb-fb-matrix}). In this problem, we are given an input matrix $X \in \{-1,+1\}^{n \times p}$ (each row $X_i$ is a user) and the goal is to estimate the sign of the columns for each consensus columns (e.g. constant column).
    
    To this end, given an algorithm $\A$ for DP-OPE and an input $X \in \{-1,+1\}^{n \times p}$, we have the following procedure for estimating the signs of the columns of $X$. We design an online experts problem that has $d = 2p$ experts where column $j \in [p]$ in $X$ will have two corresponding experts $2j$ and $2j+1$ (corresponding to the sign of column $j$). We design the loss functions in a way that will allow to estimate the signs of the columns based on the experts chosen by the algorithm. In particular, at round $t \in [T]$, we pick a random user $i \in [n]$ and define a loss function based on the user $X_i$. We will make $\ell_{t}(2j+1) = 0 $ and $\ell_{t}(2j+2) = 1 $ if the sign of the $j$'th column is $X_{ij} = -1$, and will set $\ell_{t}(2j+1) = 1 $ and $\ell_{t}(2j+2) = 0 $ if the sign is $X_{ij} = +1$. This implies that the algorithm will either estimate the correct sign of some column $j \in [p]$ or suffer $+1$ in regret. However, the algorithm can keep estimating the sign of the same column and still suffer low regret. In order to prevent this, once the algorithm has estimated the sign of a column $j$, we increase the loss of all experts corresponding to that column to $1$, hence preventing the algorithm from iteratively estimating the sign of the same column. Overall, this implies that a low-regret algorithm will be able to estimate the signs of most of the columns, which is a contradiction to lower bounds for private sign estimation.
    We formalize these ideas and complete the proof in~\cref{sec:apdx-thm-lb-adaptive-adv}.

\end{proof}

We also have the following lower bound for pure differential privacy. It shows that pure DP algorithms cannot learn against adaptive adversaries, that is, they must suffer linear regret for constant $\diffp$.
\begin{theorem}
\label{thm:lb-adaptive-adv-pure}
    Let $\diffp \le 1/10$ and $d \ge 2T$.
    If $\A$ is $\diffp$-DP then there is an adaptive adversary such that
    \begin{equation*}
        \E\left[\sum_{t=1}^T \ell_t(x_t) - \min_{x \in [d]} \sum_{t=1}^T \ell_t(x)\right]
        \ge \Omega \left(T \right).
    \end{equation*}
\end{theorem}

We can also extend the previous lower bounds to larger values of $\diffp$ (see proof in~\cref{sec:thm-lb-large-peps}).
\begin{theorem}
\label{thm:lb-large-peps}
    Let $1 \le k \le O(p^{1-\rho})$ for $0 < \rho < 1$, $d = 2^k \binom{p}{k} = 2^{\Theta(k \log p)}$, $T = p/k$, and $\diffp \le \frac{\sqrt{k/T}}{200 \rho \log(T)}$ where $p$ is sufficiently large. 
    If $\A$ is \ed-DP with $\delta \le 1/T^3$, then there is an adaptive adversary such that
    \begin{equation*}
       \E\left[ \sum_{t=1}^T \ell_t(x_t) - \min_{x \in [d]} \sum_{t=1}^T \ell_t(x) \right]
        \ge \Omega \left(\frac{\sqrt{T \log d}}{\diffp \log^{3/2} T } \right).
    \end{equation*}
\end{theorem}

Finally, we note that for stochastic adversaries, existing lower bounds for DP-SCO immediately imply lower bounds for the online setting using online-to-batch transformations. As there is a lower bound of $\log(d)/T\diffp$ for private selection~\citep{SteinkeUll17b}, this implies a lower bound on the (normalized) excess loss for DP-SCO with linear functions in $\ell_1$-geometry (as one can reduce private selection to this problem; \citealp{AsiFeKoTa21}). This implies a lower bound of $\log(d)/\diffp$ for DP-OPE. 


%% file: oco-imp.tex
\section{Implications for DP-OCO in $\ell_2$-geometry}
\label{sec:oco-imp}
In this section, we derive several implications of our techniques for \emph{differentially private online convex optimization} (DP-OCO) in $\ell_2$-geometry. In this setting, the algorithm chooses $x_t \in \mc{X}$ where $\mc{X} = \{x \in \R^d: \ltwo{x} \le D\}$ and the adversary responds with loss functions $\ell_t: \mc{X} \to \R^d$ that are convex and $L$-Lipschitz. Building on our techniques for DP-OPE, we propose new algorithms that improve over the best existing regret bounds for DP-OCO~\citep{KairouzMcSoShThXu21} which achieve $d^{1/4} \sqrt{T/\diffp}$ for stochastic and adaptive adversaries. Our algorithms obtain (up to logarithmic factors) near-optimal regret $\sqrt{T} + \sqrt{d}/\diffp$ for stochastic adversaries, and $\sqrt{Td} + T^{1/3}d/\diffp$ for adaptive adversaries.

\subsection{Oblivious adversaries}
\label{sec:oco-obl}

Using our private shrinking dartboard algorithm, in this section we develop algorithms that improve the regret for oblivious adversaries. Our algorithms construct a covering of the parameter space $\mc{X}$ then apply our private shrinking dartboard algorithm where the experts are the elements of the cover. By optimizing the size of the cover to balance the error from the approximation error and the error due to the number of experts, we obtain the following regret for DP-OCO in $\ell_2$-geometry. 
We defer the proof to~\cref{sec:apdx-thm-oco-imp}.
\begin{theorem}
\label{sec:thm-oco-imp}
    Let $\mc{X} = \{x \in \R^d: \ltwo{x} \le D\}$ and $\ell_1,\dots,\ell_T : \mc{X} \to \R$ be convex and $L$-Lipschitz functions chosen by an oblivious adversary. 
    There is an
     \ed-DP that has regret
    \begin{equation*}
     \E\left[ \sum_{t=1}^T \ell_t(x_t) - \min_{x \in \mc{X}} \sum_{t=1}^T \ell_t(x) \right]
        \le LD \cdot O \left( \sqrt{T d \log T } + \frac{T^{1/3} d \log^{1/3}(1/\delta)  \log T}{\diffp} \right)
        .
    \end{equation*}
\end{theorem}

In the high-privacy regime, this result can improve over the previous work~\citep{KairouzMcSoShThXu21} which has regret $\sqrt{T} d^{1/4}/\sqrt{\diffp}$. For example, if $d=1$ and $\diffp = T^{-1/4}$, then our regret is roughly $T^{7/12}$ while their regret is $T^{5/8}$.

\subsection{Stochastic adversaries}
\label{sec:oco-stoch}
For stochastic adversaries, we use the reduction in~\cref{sec:ub-stoch} (\cref{alg:stoch-adv}) with optimal algorithms for DP-SCO in $\ell_2$-geometry to obtain optimal regret bounds.
More precisely, we use an optimal \ed-DP-SCO algorithm from~\cite{FeldmanKoTa20}, which we call $\Aopt$. As this algorithm has excess loss $\Delta_n = LD \cdot O(1/\sqrt{n} + \sqrt{d}/n\diffp)$, the following result follows immediately from~\cref{thm:ub-stoch-OCO}. We defer the proof to~\cref{sec:apdx-cor-DP-OCO}. 
We also note that we can also obtain regret bounds for pure $\diffp$-DP using existing (pure) DP algorithms for DP-SCO~\citep{AsiLeDu21} with our reduction. 
 
\begin{corollary}[DP-OCO in $\ell_2$-geometry]
\label{cor:DP-OCO}
    Let $\mc{X} = \{ x \in \R^d: \ltwo{x} \le D\}$
    and $\ell_1,\dots,\ell_T : \mc{X} \to \R$ be convex and $L$-Lipschitz functions chosen by a stochastic adversary, $\ell_i \simiid P$.
    Then
    \cref{alg:stoch-adv} using $\Aopt$ is \ed-DP and has regret
    \begin{equation*}
     \E\left[ \sum_{t=1}^T \ell_t(x_t) - \min_{x \in [d]} \sum_{t=1}^T \ell_t(x) \right]
        \le LD \cdot O \left(\sqrt{T} + \frac{\sqrt{d} \log T}{\diffp} \right)
    \end{equation*}   
\end{corollary}
This regret is near-optimal up to logarithmic factors since we have the lower bound $\sqrt{T} + \sqrt{d}/\diffp$ for the offline version of this problem (DP-SCO in $\ell_2$-geometry) where all of the samples are given in advance~\citep{BassilySmTh14,BassilyFeTaTh19}.

%% file: appendix-ub.tex
\section{Proofs for~\cref{sec:ub-obl-sd}}
\label{sec:apdx-ub-obl}

\subsection{Proof of~\cref{thm:ub-priv-DS}}
\label{sec:proof-thm-ub-DS}
We build on the following lemma.
\begin{lemma}
\label{lemma:DS-marg-dist}
    Let $\hat P_t$ be the marginal distribution of $x_t$ of~\cref{alg:SD}. Then 
    \begin{equation*}
        \norm{\hat P_t - P^t}_{TV} \le e^{-Tp/3}.
    \end{equation*}
\end{lemma}
\begin{proof}
    Let $k_t$ be the value of $k$ at iteration $t$. We roughly show that if $k_t < K$ then $P_t = \hat P_t$. As $P(k_t>K)$ is very small, this will prove the claim. Recall that $k_t = \sum_{i \le t} \indic{z_i = 0}$. Note that $P(z_t = 0) \le p + (1-p)\eta \le 2p$. Therefore, letting $y_t \sim \mathsf{Ber}(p + (1-p)\eta)$ we have
    \begin{align*}
    P(k_t > K)
        & \le P(k_T > K) \\
        & = P(\sum_{i=1}^T \indic{z_i = 0}> K) \\
        & \le P(\sum_{i=1}^T \indic{y_i = 0}> K) \\
        & \le e^{-Tp/3},
    \end{align*}
    where the last inequality follows from a Chernoff bound~(\cref{lemma:chernoff}).
    
    Now we proceed to show that $\hat P_t$ and $P_t$ are close. To this end, we first define $Q_t$ to be the marginal distribution of $x_t$ in~\cref{alg:SD} when $K=T+1$ (that is, no limit on switching). We prove by induction that $Q_t = P^t$. The base case for $t=1$ is trivial. Assuming correctness for $t$, we have that for $x \in [d]$
    \begin{align*}
    Q_t(x)  
        & = p p_x^t + (1-p) \frac{w^t_{x}}{w^{t-1}_{x}} Q_{t-1}(x) + (1-p) p_x^t \sum_{x'=1}^d Q_{t-1}(x') (1 - \frac{w^t_{x'}}{w^{t-1}_{x'}}) \\
        & = p p_x^t + (1-p) \frac{w^t_{x}}{w^{t-1}_{x}} \frac{w_x^{t-1}}{W^{t-1}} + (1-p) \frac{w^t_x}{W^t} \sum_{x'=1}^d \frac{w_{x'}^{t-1}}{W^{t-1}}  \frac{w^{t-1}_{x'}-w^t_{x'}}{w^{t-1}_{x'}} \\
        & = p p_x^t + (1-p) \left( \frac{w^t_{x}}{W^{t-1}} +  \frac{w^t_x}{W^t} \frac{W^{t-1} - W^t}{W^{t-1}} \right) \\
        & = p_x^t.
    \end{align*}
    Now consider $\hat P$. Let $Q_t^0$ and $Q_t^1$ be the conditional distribution of $Q_t$ given $k_t<K$ or $k_t \ge K$, respectively.  Moreover, let $\hat P_t^0$ and $\hat P_t^1$ be the conditional distribution of $\hat P_t$ given $k_t<K$ or $k_t \ge K$, respectively. Note that $Q_t(x) =  P(k_t<K) Q_t^0 +  P(k_t<K) Q_t^1$ and that $\hat P_t(x) =  P(k_t<K) \hat P_t^0 +  P(k_t<K) \hat P_t^1$. Noting that $P_t^0 = Q^t_0$, we have
    \begin{align*}
    \norm{\hat P_t - P^t}_{TV}
        & = \norm{\hat P_t - Q^t}_{TV} \\
        & = \norm{P(k_t<K)(\hat P_t^0 - Q_t^0) + P(k_t>K)(\hat P_t^1 - Q_t^1) }_{TV} \\
        & \le  P(k_t<K) \norm{\hat P_t^0 - Q_t^0}  + P(k_t>K) \norm{\hat P_t^1 - Q_t^1}_{TV} \\
        & \le  P(k_t>K).
    \end{align*}
\end{proof}

    

\begin{proof}
    First, we begin by analyzing the regret. \cref{lemma:DS-marg-dist} shows that $\hat P_t$ the marginal distribution of $x_t$is the same as that of the (non-private) shrinking dartboard algorithm $P_t$, therefore Theorem 3 of \citet{GeulenVoWi10} shows that for $\eta \le 1/2$
    \begin{align*}
     \E_{x_t \sim \hat P_t}\left[ \sum_{t=1}^T \ell_t(x_t)\right]
        & = \E_{x_t \sim P^t}\left[ \sum_{t=1}^T \ell_t(x_t)\right] 
         +  \E_{x_t \sim \hat P_t}\left[ \sum_{t=1}^T \ell_t(x_t)\right] -  \E_{x_t \sim P^t}\left[ \sum_{t=1}^T \ell_t(x_t)\right] \\
        & \le E_{x_t \sim P^t}\left[ \sum_{t=1}^T \ell_t(x_t)\right]  + 2 T \norm{\hat P_t - P^t}_{TV} \\
        & \le (1+ \eta) \min_{x \in [d]} \sum_{t=1}^T \ell_t(x) + \frac{\ln d}{\eta} + 2 T e^{-Tp/3} \\
        & \le  \min_{x \in [d]} \sum_{t=1}^T \ell_t(x) + \eta T
        + \frac{\ln d}{\eta} + 2 T e^{-Tp/3}. 
    \end{align*}

    Let us now analyze privacy. Assume we have two neighboring sequences that differ at time-step $t_1$. 
    Let $Z_t$ and $X_t$ denote the random variables for $z_t$ and $x_t$ in the algorithm when run for the first sequence and let $Y_t = 1 - Z_t$. Similarly, let $Z'_t$, $Y'_t$, and $X'_t$ denote the same for the neighboring sequence. We consider the pairs $W_t = (X_t,Z_{t+1})$ (where $X_0 = 0$) and prove that  $W_t$  given $\{ W_\ell \}_{\ell=0}^{t-1}$ and $W'_t$  given $\{ W'_\ell \}_{\ell=0}^{t-1}$  are $\diffp_t$-indistinguishable where 
    \begin{equation*}
        \diffp_t = 
        \begin{cases}
             0  & \text{if } t < t_1 \\
             \eta/p & \text{if } t = t_1 \\
             \indic{\sum_{\ell=1}^{t-1} Y_\ell < K} 4 Y_t \eta &\text{if } t > t_1
        \end{cases}
    \end{equation*}
    The result then follows from advanced composition~\citep{DworkRo14}: note that $Y_t \in \{0,1\}$ therefore we have that the final privacy parameter is 
    \begin{align*}
    \diffp_f 
        & \le \frac{3}{2} \sum_{t=1}^T \diffp_t^2 + \sqrt{6 \sum_{t=1}^T \diffp_t^2 \log(1/\delta) } \\
        & \le \frac{3}{2} (\frac{\eta^2}{p^2} + 16 K \eta^2) + \sqrt{6(\frac{\eta^2}{p^2} + 16K \eta^2)\log(1/\delta)  } \\
        & \le  \frac{5\eta}{p} + 24 K \eta^2 + \eta \sqrt{100 K \log(1/\delta)} \\
         & \le  \frac{5\eta}{p} + 100 T p  \eta^2 + 20 \eta \sqrt{ T p \log(1/\delta)}.
    \end{align*}
    Similarly, the result for $\delta=0$ follows from basic composition.
    To finish the proof, consider the pair $W_t$ and $W'_t$. First, note that if $t<t_1$ then clearly $W_t$ and $W'_t$ are $0$-indistinguishable as they do not depend on $\ell_{t_1}$ or $\ell'_{t_1}$. For $t=t_1$, note that $X_{t_1}$ and $X'_{t_1}$ has the same distribution. Moreover, the definition of $Z_t$ implies that 
    $Z_t$ and $Z'_t$ are $\eta/p$-indistinguishable since
    \begin{align*}
    \frac{P(Z_t = 1)}{P(Z'_t = 1)}
        & \le \frac{(1-p)}{(1-p)(1-\eta)} \\
        &  = \frac{1}{1 - \eta} \\
        &  = 1 + \frac{\eta}{1-\eta} \\
        &  \le 1 + 2 \eta \\
        & \le e^{2\eta}. 
    \end{align*}
    Moreover, since $\eta \le p \diffp  $ we have
    \begin{align*}
    \frac{P(Z_t = 0)}{P(Z'_t = 0)}
        & \le \frac{p + (1-p)\eta}{p} \\
        &  \le 1 + \frac{\eta}{p} 
        \le e^{\eta/p}.
    \end{align*}
    Now consider $t>t_1$. If $\sum_{\ell=1}^{t-1} Y_\ell \ge K$ or $Y_t = Y'_t = 0$ then $X_t = X_{t-1}$ and $X'_t = X'_{t-1}$ and thus $X_t$ and $X'_t$ are $0$-indistinguishable.
    If $Y_t = Y'_t = 0$ then $X_t$ and $X'_t$ are $4\eta$-indistinguishable since $w_x^t/w_x^{'t} \le 1/(1-\eta) \le e^{2\eta}$ which implies that $P(x_t = x)/P(x'_t=x) \le e^{4\eta}$. Overall, $X_t$ and $X'_t$ are $4Y_t \eta$-indistinguishable. Moreover, since $t>t_1$, we have that $Z_{t+1}$ is a function of $X_t$ and $\ell_t$  and $Z'_{t+1}$ is a function of $X'_t$ and $\ell'_t=\ell_t$, hence by post-processing we get that $Z_{t+1}$ and $Z'_{t+1}$ are $4Y_t \eta$-indistinguishable. Overall, we have that $W_t$ and $W'_t$ are $\indic{\sum_{\ell=1}^{t-1} Y_\ell < K} 4 Y_t \eta$-indistinguishable.

\end{proof}

\subsection{Proof of~\cref{cor:sd-appr}}
\label{sec:apdx-cor-sd-appr} 
For these parameters, \cref{alg:SD} has privacy
\begin{equation*}
    \diffp_0/4 +  T p^3 \diffp_0^2/4 +  \diffp_0  \sqrt{T p^3\log(1/\delta)} \le 2\diffp_0.
\end{equation*}
As $\diffp_0 \le \diffp/2$, this proves the claim about privacy.
Moreover, its regret is
\begin{align*}
     \eta T
        + \frac{\ln d}{\eta} + 2 T e^{-Tp/3} 
    & \le T p \diffp_0/20 + 20 \ln d /(p\diffp_0) + 2 T e^{-Tp/3} \\
    & \le \frac{T^{2/3} \diffp_0}{\log^{1/3}(1/\delta)} + \frac{20 T^{1/3} \log^{1/3}(1/\delta) \ln d}{\diffp_0} + 2 T e^{-Tp/3} \\
    & \le \sqrt{T \ln d } + \frac{20 T^{1/3} \log^{1/3}(1/\delta) \ln d}{\diffp_0} + 2 T e^{-Tp/3} \\
    & \le O \left(\sqrt{T \ln d } + \frac{ T^{1/3} \log^{1/3}(1/\delta) \ln d}{\diffp} \right),
\end{align*}
where the last inequality follows as $\diffp_0 = \min( \diffp/2,\frac{\log^{1/3}(1/\delta) \sqrt{\ln d}}{T^{1/6}})$.

\subsection{Proof of~\cref{cor:pure}}
\label{sec:apdx-cor-pure} 
For these parameters, \cref{alg:SD} has privacy
\begin{equation*}
    \diffp/20 +  16 Tp \eta 
    \le \diffp/10 +  16Tp^2 \diffp/20
    \le \diffp.
\end{equation*}
Moreover, its regret is
\begin{align*}
     \eta T
        + \frac{\ln d}{\eta} + 2 T e^{-Tp/3} 
    & \le T p \diffp/20 + 20 \ln d /(p\diffp) + 2 T e^{-Tp/3} \\
    & \le \sqrt{T} + \frac{20 \sqrt{T} \ln d}{\diffp} + 2 T e^{-Tp/3},
\end{align*}
where the last inequality follows since $\diffp \le 1$.

\subsection{Proof of~\cref{cor:sd-batch}}
\label{sec:adpx-cor-sd-batch}
\iftoggle{arxiv}{}{
To prove~\cref{cor:sd-batch}, we first prove the following proposition which charactarizes the performance of the private shrinking dartboard algorithm with batches. We prove this result in~\cref{sec:apdx-thm-ub-priv-DS-batch}.
\begin{theorem}
\label{thm:ub-priv-DS-batch}
    Let $\ell_1,\dots,\ell_T \in [0,1]^d$ be chosen by an oblivious adversary. \cref{alg:SD} with batch size $1 \le B \le T$, $p < 1/2$, $\eta<1/2$, and $K = 2 T p/B $ has regret
    \begin{equation*}
     \E\left[ \sum_{t=1}^T \ell_t(x_t) - \min_{x \in [d]} \sum_{t=1}^T \ell_t(x) \right]
        \le \eta T
        + \frac{B \ln d}{\eta} + 2 T e^{-Tp/3B}.
    \end{equation*}
    Moreover, for $\delta>0$, \cref{alg:SD} is $(\diffp,\delta)$-DP where 
    \begin{equation*}
     \diffp = 
          \frac{5\eta}{Bp} + 100 T p  \eta^2/B^3 + \frac{20\eta}{B} \sqrt{12 T p/B \log(1/\delta)}. 
    \end{equation*}
\end{theorem}

We are now ready to prove~\cref{cor:sd-batch}.
}

For these parameters, \cref{alg:SD} has privacy
\begin{equation*}
    \frac{5\eta}{Bp} + \frac{100 T p  \eta^2}{B^3} + \frac{20\eta}{B^{3/2}} \sqrt{ T p \log(1/\delta)} 
    \le \diffp/8 +  \frac{T p^3 \diffp^2}{16B} +   \frac{\diffp}{2} \sqrt{ T p^3 \log(1/\delta)/B} 
    \le \diffp.
\end{equation*}
Moreover, its regret is
\begin{align*}
     \eta T
        + \frac{B\ln d}{\eta} + 2 T e^{-Tp/3B} 
    & \le T B p \diffp/40 + 40 \ln d /(p\diffp) + 2 T e^{-Tp/3B} \\
    & \le \frac{T^{2/3} B^{4/3} \diffp}{\log^{1/3}(1/\delta)} + \frac{40T^{1/3} \log^{1/3}(1/\delta) \ln d}{B^{1/3}\diffp} + 2 T e^{-Tp/3B} \\
    & \le O \left( \frac{T^{2/5} \log^{1/5}(1/\delta) \log^{4/5}(d))  }{\diffp^{4/5}}  + 2 T e^{-Tp/3B} \right),
\end{align*}
where the last inequality follows by choosing $B = \frac{\log^{2/5}(1/\delta) \log^{3/5}(d)}{T^{1/5} \diffp^{3/5}} $ (note that $B \ge 1$ for $\diffp \le  \frac{\log^{2/3}(1/\delta) \log(d)}{T^{1/3}}$) and noticing that $Tp/B \ge \Omega(T^{2/5})$ for these parameters as we have a lower bound on $\diffp$.

\iftoggle{arxiv}{
\subsection{Proof of~\cref{thm:ub-priv-DS-batch}}
\label{sec:apdx-thm-ub-priv-DS-batch}
}{
\subsection{Proof of~\cref{thm:ub-priv-DS-batch}}
\label{sec:apdx-thm-ub-priv-DS-batch}
}
    The same analysis as in~\cref{thm:ub-priv-DS} yields regret
    \begin{equation*}
     \E\left[ \sum_{t=1}^{\tilde T} \tilde \ell_t(\tilde x_t) - \min_{x \in [d]} \sum_{t=1}^T \tilde  \ell_t(x) \right]
        \le \eta \tilde T
        + \frac{\ln d}{\eta} + 2 \tilde  T e^{-\tilde Tp/3}.
    \end{equation*}
    Setting $x_t = \tilde x_{\floor{t/B}}$ and multiplying both sides by $B$, we have regret
    \begin{equation*}
     \E\left[ \sum_{t=1}^{ T}  \ell_t(x_t) - \min_{x \in [d]} \sum_{t=1}^T  \ell_t(x) \right]
        \le \eta  T
        + \frac{B \ln d}{\eta} + 2  T e^{-\tilde Tp/3}.
    \end{equation*}
    
    Let us now analyze privacy. The privacy follows the same steps as in the proof of~\cref{thm:ub-priv-DS} with two main differences. First, let $t=t_1$ be the time such that $\tilde \ell_{t_1}$ contains the differing loss function and let $\tilde \ell_{t} = \tilde \ell_{t}(x_{t-1})$ and $\tilde \ell'_{t} = \tilde \ell'_{t}(x_{t-1})$. Note that $|\tilde \ell_{t_1} - \tilde \ell'_{t_1}| \le 1/B$ thus we have that $Z_t$ and $Z'_t$ are $\eta/(Bp)$-indistiguishable since
     \begin{align*}
    \frac{P(Z'_t = 1)}{P(Z_t = 1)}
        & \le \frac{(1-p)(1-\eta)^{\tilde \ell'_{t-1}}}{(1-p)(1-\eta)^{\tilde \ell_{t-1}}} \\
        &  \le  (1-\eta)^{-|\tilde \ell_{t-1} - \tilde \ell'_{t-1}|} \\
        &  \le e^{2\eta/B}.
    \end{align*}
    Moreover, assuming w.l.o.g. that $\tilde \ell'_{t-1} \ge \tilde \ell_{t-1}$, we have
    \begin{align*}
    \frac{P(Z'_t = 0)}{P(Z_t = 0)}
        & \le \frac{p + (1-p)(1 - (1-\eta)^{\tilde \ell'_{t-1}})}{p + (1-p)(1 - (1-\eta)^{\tilde \ell_{t-1}}) } \\
        &  \le 1 + \frac{(1-p)|1 - (1-\eta)^{\tilde \ell'_{t-1} - \tilde \ell_{t-1}}|}{p + (1-p)(1 - (1-\eta)^{\tilde \ell_{t-1}}) } \\
        & \le 1 + \frac{|1 - (1-\eta)^{\tilde \ell'_{t-1} - \tilde \ell_{t-1}}|}{p} \\
        & \le 1 + \frac{|{\tilde \ell'_{t-1} - \tilde \ell_{t-1}}|}{p} \le e^{\eta/(Bp)}.
    \end{align*}
    The second difference in the privacy analysis is that the sensitivity of the score of the exponential mechanism is now $1/B$ hence $X_t$ and $X'_t$ are now $4\eta/B$-DP. This shows that  $W_t$  given $\{ W_\ell \}_{\ell=0}^{t-1}$ and $W'_t$  given $\{ W'_\ell \}_{\ell=0}^{t-1}$  are $\diffp_t$-indistinguishable where 
    \begin{equation*}
        \diffp_t = 
        \begin{cases}
             0  & \text{if } t < t_1 \\
             \eta/(Bp) & \text{if } t = t_1 \\
             \indic{\sum_{\ell=1}^{t-1} Y_\ell < K} 4 Y_t \eta/B &\text{if } t > t_1
        \end{cases}
    \end{equation*}
    The result then follows from advanced composition~\citep{DworkRo14}: the final privacy parameter is 
    \begin{align*}
    \diffp_f 
        & \le \frac{3}{2} \sum_{t=1}^T \diffp_t^2 + \sqrt{6 \sum_{t=1}^T \diffp_t^2 \log(1/\delta) } \\
        & \le \frac{3}{2} (\frac{\eta}{Bp} + 16K \eta^2/B^2) + \sqrt{6(\frac{\eta^2}{B^2p^2} + 16K \eta^2/B^2)\log(1/\delta)  } \\
        & \le  \frac{5\eta}{Bp} + 24 K \eta^2/B^2 + \frac{10\eta}{B} \sqrt{K \log(1/\delta)} \\
         & \le  \frac{5\eta}{Bp} + 100 T p  \eta^2/B^3 + \frac{20\eta}{B} \sqrt{12 T p/B \log(1/\delta)}.
    \end{align*}

\section{Proofs for~\cref{sec:ub-stoch}}
\label{sec:apdx-ub-stoch}

\subsection{Proof of~\cref{thm:ub-stoch-OCO}}
\label{sec:thm-ub-stoch-OCO}
The privacy claim is immediate as each sample $\ell_i$ is used only once in running a single \ed-DP algorithm. Now we prove the claim about utility. Consider time-step $t=2^i$ where we invoke a DP-SCO algorithm with $t/2 = 2^{i-1}$ samples. 
Therefore the guarantees of the algorithm imply that at iteration $t$ we have
\begin{equation*}
        \E_{\ell_t \sim P} \left[\ell_t(x_t) - \min_{x \in [d]} \ell_t(x) \right] \le O \left( \Delta_{2^i} \right).
\end{equation*}
Therefore at phase $i$, that is $2^{i} \le t \le 2^{i+1}$, the total regret is at most 
\begin{equation*}
    \E\left[ \sum_{t=2^{i}}^{2^{i+1}}\ell_t(x_t) - \min_{x \in [d]} \sum_{t=2^{i}}^{2^{i+1}} \ell_t(x) \right] 
        \le O \left(2^i \Delta_{2^i} \right).
\end{equation*}
Summing over $i$ proves the claim.

\subsection{Proof of~\cref{cor:DP-exp-stoch}}
\label{sec:apdx-cor-DP-exp-stoch}
The algorithm $\Aone$ is $\diffp$-DP and has excess population loss $\Delta_n = O(\sqrt{\log(d)/n} + \log(d)/n\diffp)$~\cite[Theorem 6]{AsiFeKoTa21}. Thus, \cref{thm:ub-stoch-OCO} implies that 
\begin{align*}
     \E\left[ \sum_{t=1}^T \ell_t(x_t) - \min_{x \in [d]} \sum_{t=1}^T \ell_t(x) \right]
        & \le \sum_{i=1}^{\log T} 2^i \Delta_i \\
        & \le O \left( \sum_{i=1}^{\log T} 2^{i/2} \sqrt{\log(d)} + \log(d)/\diffp \right) \\
        & \le O \left( \sqrt{T \log(d)} + \log(d) \log(T)/\diffp \right). 
\end{align*}

%% file: appendix-LB.tex
\section{Proofs for~\cref{sec:lower-bounds}}

\subsection{Proof of~\cref{thm:lb-adaptive-adv}}
\label{sec:apdx-thm-lb-adaptive-adv}

We build on the following property of the padded Tardos code as done in finger-printing lower bounds.
Given a matrix $X \in \{-1,+1\}^{(n+1) \times p}$, we say that $j \in [p]$ is a consensus column if the column is equal to the all one vector or its negation. Let $X_{(i)} \in \{-1,+1\}^{n \times p}$ denote the matrix that results from removing the $i$'th row in $X$. Moreover, we let $\bar X \in \R^p$ denote the sum of the rows of $X$, that is, $\bar X_j = \sum_{i=1}^{n+1} X_{ij}$. Finally, for $v \in \R^p$ let $\sign(v) \in \{-1,+1\}^p$ denote the signs of the entries of $v$
\begin{theorem}[{\citealp[Theorem 3.2]{TalwarThZh15}}]
\label{thm:lb-fb-matrix}
    Let $p = 1000m^2$ and $n = m/\log m$ for sufficiently large $m$. There exists a matrix $X \in \{-1,+1\}^{(n+1) \times p}$ such that 
    \begin{itemize}
        \item There are at least $0.999p$ consensus columns in $X_{(i)}$
        \item Any algorithm $\A: \{-1,+1\}^{n \times p} \to \{-1,+1\}^p$ such that $\lzero{\A(X_{(i)}) - \mathsf{sign}(\bar X_{(i)})} \le 1/4$ for all $i\in[n+1]$ with probability at least $2/3$ then $\A$ is not $(1,n^{-1.1})$-DP. 
    \end{itemize}
\end{theorem}


Building on~\cref{thm:lb-fb-matrix}, we can now prove our main lower bound.

\newcommand{\Scons}{S_{\mathsf{cons}}}
\begin{proof}[of \cref{thm:lb-adaptive-adv}]
    First, we prove the lower bound for $\diffp \le 1/(\sqrt{T} \log T)$, that is, we prove the regret has to be linear in this case.
    We will reduce the problem of private sign estimation to DP-OPE with adaptive adversaries and use the lower bound of~\cref{thm:lb-fb-matrix}. To this end, given an algorithm $\A$ for DP-OPE and an input $X \in \{-1,+1\}^{n \times p}$, we have the following procedure for estimating the signs of the columns of $X$. We design an online experts problem that has $d = 2p$ experts where column $j \in [p]$ in $X$ will have two corresponding experts $2j$ and $2j+1$ (corresponding to the sign of column $j$). We initialize the vector of signs $s_j = 0$ for all $1 \le j \le p$. We have $T=0.9p$ rounds and at round $1 \le t \le T$ we sample a user $i_t \sim [n]$ (arbitrarily while enforcing that each $i \in [n]$ appears at most $2T/n$ times) and play a loss function $\ell_t: [d] \to  \{ 0,1 \}$ such that
    \begin{equation*}
        \ell_{t}(2j+1) = 
        \begin{cases}
            1 & \text{if } s_j \neq 0 \\
            \frac{X_{i_t,j}+1}{2} & \text{otherwise}
        \end{cases}
    \end{equation*}
    We also set  
    \begin{equation*}
        \ell_{t}(2j+2) = 
        \begin{cases}
            1 & \text{if } s_j \neq 0 \\
            \frac{-X_{i_t,j}+1}{2} & \text{otherwise}
        \end{cases}
    \end{equation*}
    The idea of this loss function is that the $2j+1$ and $2j+2$ experts will represent the signs of the $j$'th column. If the sign of the $j$'th column is $+1$, then expert $2j+2$ will have better loss and hence should be picked by the algorithm. Moreover, whenever the algorithm has estimated the sign of the $j$'th column ($s_j \neq 0$), we set the loss to be $1$ for both experts $2j+1$ and $2j+2$, in order to force the online algorithm to estimate the sign of new columns. 
    
    Then, given the output of the algorithm $\A$ at time $t$, that is $x_t = \A(\ell_1,\dots,\ell_{t-1})$ we set $s_j = -1$ if $x_{t} = 2j+1$ and $s_j = 1$ if $x_{t} = 2j+2$ and otherwise we keep $s_j$ unchanged. 
    Moreover, there is an expert that achieves optimal loss, that is, for some $x\opt \in [d]$ we have
    \begin{equation*}
        \sum_{t=1}^T \ell_t(x\opt) = 0.
    \end{equation*}
    This follows since $X$ has at least $0.999p$ consensus columns hence there is a zero-loss expert after $T=0.9p$ iterations. Now we show that if an algorithm $\A$ has small regret, then the vector $s$ estimates the sign of at least $0.8p$ columns. To this end, let $ j_t = \floor{x_t/2}$ denote the column corresponding to the expert picked by the algorithm at time $t$, $S = \{j_t : t \in[T] \}$, and $\Scons = \{ j \in [p]: \text{column j is a consensus column} \} $. Observe that the regret of the algorithm is
    \begin{align*}
        \sum_{t=1}^T \ell_t(x_t) 
        & = \sum_{j_t \in S} \indic{s_{j_t}=1} \ell_t(2j_t+2)
        +  \indic{s_{j_t}=-1} \ell_t(2j_t+1)
        \\ 
        & = \frac{1}{2} \sum_{j_t \in S} \indic{s_{j_t}=1} (-X_{i_t,j_t}+1)
        +  \indic{s_{j_t}=-1} (X_{i_t,j_t}-1) \\
        & = \frac{1}{2} \sum_{j_t \in S} \indic{s_{j_t}=1, X_{i_t,j_t} = -1}
        +  \indic{s_{j_t}=-1, X_{i_t,j_t} = 1} \\
        & = \frac{1}{2} \sum_{j_t \in S} \indic{s_{j_t} \neq X_{i_t,j_t}} \\
        & \ge -0.001 p + \frac{1}{2} \sum_{j_t \in S \cap \Scons} \indic{s_{j_t} \neq X_{i_t,j_t}} \\
        & \ge -0.001 p + \frac{1}{2} \sum_{j_t \in S \cap \Scons} \indic{s_{j_t} \neq \sign(\bar X)_{j_t}}.
    \end{align*}
    Assume towards a contradiction that $\A$ is $(1/200\sqrt{T}\log(T),\delta)$-DP where $\delta \le 1/T^3$ and that the expected regret is at most $T/1000$. Markove inequality implies that with probability at least $9/10$ the regret is at most $T/100$.  Under this event we have
    \begin{equation*}
     \sum_{j_t \in S \cap \Scons} \indic{s_{j_t} \neq \sign(\bar X)_{j_t}} \le 0.002 T.    
    \end{equation*}
    Now note that we can assume that the online algorithm picks $x_t$ such that each $j_t$ appears at most one. Otherwise we can modify the algorithm to satisfy this property while not increasing the regret: whenever the algorithm picks $x_t$ such that $j_t$ appeared before, the loss of this expert is $1$, hence we can randomly pick another expert $x_t$ such that $j_t$ has not appeared. This implies that $|S| = T = 0.9p$ and hence 
    $|S \cap \Scons| \ge 0.85p$. Therefore we have that $s_{j_t} = \sign(\bar X)_{j_t}$ for at least $0.8p$ columns from $\Scons$ with probability $0.9$. To finish the proof, we need to argue about the final privacy guarantee of the sign vector $s$; we will prove that $s$ is $(1,T\delta)$-DP which will give a contradiction to~\cref{thm:lb-fb-matrix} and prove the claim. To this end, note that the algorithm $\A$ is $(1/200\sqrt{T}\log(T),\delta)$-DP. Moreover, recall that each row $i\in[n]$ appears at most $k \le 2T/n \le 2p/n \le 200 \sqrt{p} \log(p)$ times, hence group privacy implies the final output $s$ is $(k\diffp,k\delta)$-DP, that is, $(1,1/T^2)$-DP.
    
    Now we proceed to prove the lower bound for larger values $\diffp \ge 1/(\sqrt{T} \log T)$. Note that if $\diffp \ge \log(T)/T^{1/4}$ then the non-private lower bound of $\sqrt{T \log d}$ is sufficient. Otherwise, consider an algorithm $\A$ that is $\diffp$-DP and consider an adversary that in the first $T_0 < T$ iterations behaves the same as the above where $\diffp = 1/(\sqrt{T_0} \log T_0)$. Then in the last $T - T_0$ iterations it sends $\ell_t(x)=0$ for all $x \in [d]$. The above lower bound implies that the algorithm has to pay regret $\Omega(T_0)= \Omega(1/(\diffp \log T_0)^2)$. The claim follows as $T_0 \le T$.

\end{proof}

\subsection{Proof for~\cref{thm:lb-adaptive-adv-pure}}
\label{sec:proof-lb-pure}

To prove a lower bound for pure DP, we use the following version of~\cref{thm:lb-fb-matrix} for this setting.
\begin{theorem}[{\citealp[Theorem A.1]{SteinkeUl17}}]
\label{thm:pure-sign-est}
    Let $d = 1000n$ and $n$ sufficiently large. Let $\mc{X} = \{X \in \{-1,+1\}^{n \times d} :$ all the columns in $X$ are consensus columns$\} $. Let $\A: \{-1,+1\}^{n \times d} \to \{-1,+1\}^d$ be an algorithm such that for all $X \in \mc{X}$,
    \begin{equation*}
      \E[\lzero{\A(X) - \mathsf{sign}(X)}] \le 1/4.  
    \end{equation*}
     Then $\A$ is not $1$-DP.
\end{theorem}

Using the bound of~\cref{thm:pure-sign-est} and following the same steps as in the proof of~\cref{thm:lb-fb-matrix}, the lower bound of~\cref{thm:lb-adaptive-adv-pure} now follows.

\subsection{Proof of~\cref{thm:lb-large-peps}}
\label{sec:thm-lb-large-peps}


We use similar ideas to the one in the proof of~\cref{thm:lb-adaptive-adv} where we used a DP-OPE algorithm for sign estimation. Instead of designing two experts for each column, the idea here is to look at subsets of columns of size $k$ and design $2^k$ experts to represent the sign vector of these $k$ columns.

Given an input $X \in \{-1,+1\}^{n \times p}$ where we assume for simplicity that $p/k$ is an integer, we design an expert problem with $d = 2^k \binom{p}{k}  $ experts. 
Instead of representing the experts as integers $x \in [d]$, we use an equivalent representation where an expert is a pair $(S,v)$ where $S \subset [p]$ is a set of columns of size $k$ and $v \in \{-1,+1\}^k$ represents the signs that this expert assigns for columns in $S$. We initialize the vector of signs $s_j = 0$ for all $1 \le j \le p$.

Here we have $T=0.9p/k$ rounds and at round $1 \le t \le T$ we sample a user $i_t \sim [n]$ (arbitrarily while enforcing that each $i \in [n]$ appears at most $2T/n$ times) and play a loss function $\ell_t$ such that
\begin{equation*}
        \ell_{t}(S,v) = 
        \begin{cases}
            1 & \text{if } s_j \neq 0 \text{ for some } j \in S \\
            0 & \text{otherwise if } \sign(\bar X_S) = v \\
            1 & \text{otherwise}
        \end{cases}
\end{equation*}
 Now, given the output of the algorithm $\A$ at time $t$, that is $x_t = (S_t,v_t)$ we set $s_{S_t} = v_t$ (we assume without loss of generality that each $j \in [p]$ will appear in at most a single $S_t$. Otherwise, similarly to the proof of~\cref{thm:lb-adaptive-adv}, we can ensure this property while not increasing the regret).
    Moreover, at the end of the game, there is a set $S \subset [p]$ of size $k$ that contains only consensus columns which were not estimated earlier ($S \cap S_t = \emptyset$ for all $t$). This follows from the fact that $X$ has at least $0.999p$ consensus columns hence there is at least $0.05p \ge k$  consensus columns that have not appeared in $S_1,\dots,S_T$, hence there is an expert $(S,v)$ such that 
    \begin{equation*}
        \sum_{t=1}^T \ell_t(S,v) = 0.
    \end{equation*}
    Now we show that if an algorithm $\A$ has small regret, then the vector $s$ estimates the sign of at least $0.8p$ columns. Observe that the regret of the algorithm is
    \begin{align*}
        \sum_{t=1}^T \ell_t(x_t) 
         & = \sum_{t=1}^T  \ell_t(S_t,v_t) 
        \\ 
        & = \sum_{t=1}^T \indic{\sign(\bar X_{S_t}) \neq v_t}
        \\ 
        & = \sum_{t=1}^T \indic{\sign(\bar X_{S_t}) \neq s_{S_t}}.
    \end{align*}
    Assume towards a contradiction that $\A$ is $(\diffp,\delta)$-DP where $\diffp \le  \frac{\sqrt{k/T}}{200\log(T)}$ and $\delta \le 1/T^3$ and that the expected regret is at most $T/1000$. Markov inequality implies that with probability at least $9/10$ the regret is at most $T/100$. 
    Note that $|S| = kT = 0.9p$.  Under this event we have
    \begin{equation*}
     \sum_{t=1}^T \indic{\sign(\bar X_{S_t}) \neq s_{S_t}}  \le 0.002 T.    
    \end{equation*}
    Hence we have that $\sign(\bar X_{S_t}) = s_{S_t} $ for at least $0.9T$ rounds. As each round has $k$ distinct columns, we have $s_j = \sign(\bar X_j)$ for at least $0.9kT \ge 0.8 p$. As there are at most $0.001p$ non-consensus columns, this means that $s_j = \sign(\bar X_j)$ for at least $0.75p$ consensus columns. Now we prove that $s$ is also $(1,1/T^2)$-DP which gives a contradiction to~\cref{thm:lb-fb-matrix}. To this end, note that the algorithm $\A$ is $(\diffp,\delta)$-DP where $\diffp \le  \frac{\sqrt{k/T}}{200\log(T)} \le  \frac{k/\sqrt{p}}{200 \rho \log(p)} $. Moreover, recall that each row $i\in[n]$ appears at most $k_i \le 2T/n \le 2 p /(nk) \le 200 \sqrt{p} \log(p)/k$ times, hence group privacy implies the final output $s$ is $(\max_{i} k_i \diffp,T \delta)$-DP, that is, $(1,1/T^2)$-DP.

%% file: appendix-oco-imp.tex
\section{Proofs for~\cref{sec:oco-imp}}

\subsection{Proof of~\cref{sec:thm-oco-imp}}
\label{sec:apdx-thm-oco-imp}

We assume without loss of generality that $L=1$ (otherwise divide the loss by $L$).
As $\mc{X}$ has diameter $D$, we can construct a cover $C = \{c_1,\dots,c_M\}$ of $\mc{X}$ such that $\min_{i \in [M]}\ltwo{x - c_i} \le \rho$ for all $x \in \mc{X}$ where $M \le 2^{d \log(4/\rho)}$~\citep[Lemma 7.6]{Duchi19}. Consider the following algorithm: run~\cref{alg:SD} where the experts are the elements of the cover $C$. \cref{cor:sd-appr} now implies that this algorithm has regret
\begin{equation*}
     \E\left[ \sum_{t=1}^T \ell_t(x_t) - \min_{x \in C} \sum_{t=1}^T \ell_t(x) \right]
     \le O \left( \sqrt{T \ln M } + \frac{T^{1/3} \log^{1/3}(1/\delta) \ln M}{\diffp}  \right).
\end{equation*}
Since $\ell_t$ is $1$-Lipschitz, we now get
\begin{align*}
     \E\left[ \sum_{t=1}^T \ell_t(x_t) - \min_{x \in \mc{X}} \sum_{t=1}^T \ell_t(x) \right]
     & \le \E\left[ \sum_{t=1}^T \ell_t(x_t) - \min_{x \in C} \sum_{t=1}^T \ell_t(x) + \min_{x \in C} \sum_{t=1}^T \ell_t(x)  - \min_{x \in \mc{X}} \sum_{t=1}^T \ell_t(x) \right] \\
     & = O \left( \sqrt{T \ln M } + \frac{T^{1/3} \log^{1/3}(1/\delta) \ln M}{\diffp}  + T\rho \right) \\
     & = O \left( \sqrt{T d \log(1/\rho) } + \frac{T^{1/3} \log^{1/3}(1/\delta) d \log(1/\rho)}{\diffp}  + T\rho \right) \\
     & = O \left( \sqrt{T d \log(T) } + \frac{T^{1/3} d \log^{1/3}(1/\delta)  \log(T)}{\diffp} \right),
\end{align*}
where the last inequality follows by setting $\rho = 1/T$.

\subsection{Proof of~\cref{cor:DP-OCO}}
\label{sec:apdx-cor-DP-OCO}
The algorithm $\Aopt$ is \ed-DP and has excess loss $\Delta_n = LD \cdot O(1/\sqrt{n} + \sqrt{d}/n\diffp)$. Thus, \cref{thm:ub-stoch-OCO} implies that 
\begin{align*}
     \E\left[ \sum_{t=1}^T \ell_t(x_t) - \min_{x \in [d]} \sum_{t=1}^T \ell_t(x) \right]
        & \le \sum_{i=1}^{\log T} 2^i \Delta_i \\
        & \le O(LD) \sum_{i=1}^{\log T} 2^{i/2} + \sqrt{d}/\diffp \\
        & \le LD \cdot O(\sqrt{T} + \sqrt{d} \log(T)/\diffp).
\end{align*}

%% file: appendix-chernoff.tex
\section{Concentration for sums of geometric variables}

In this section, we proof a concentration result for the sum of geometric random variables, which allows us to upper bound the number of switches in the sparse-vector based algorithm.
We say that $Z$ is geometric random variable with success probability $p$ if $P(W=k) = (1-p)^{k-1}p$ for $k\in\{1,2,\dots\}$. To this end, we use the following Chernoff bound.
\begin{lemma}[\citealp{MitzenmacherUp05}, Ch.~4.2.1]
  \label{lemma:chernoff}
  Let $X = \sum_{i=1}^n X_i$ for $X_i \simiid \mathsf{Ber}(p)$.
  Then for $\delta \in [0,1]$,
  \begin{align*}
    \P(X > (1+\delta)np ) \le e^{-np\delta^2 /3}
    ~~~ \mbox{and} ~~~
    \P(X < (1-\delta)np ) \le e^{-np\delta^2 /2}.
  \end{align*}
\end{lemma}

The following lemma demonstrates that the sum of geometric random variables concentrates around its mean with high probability.
\begin{lemma}
\label{lemma:geom-concentration}
Let $W_1,\dots,W_n$ be iid geometric random variables with success probability $p$. Let $W = \sum_{i=1}^n W_i$. Then for any $k \ge n$ 
\begin{equation*}
    \P(W > 2k/p ) \le \exp{\left(-k/4\right)}.
\end{equation*}
\end{lemma}
\begin{proof}
    Notice that $W$ is distributed according to the negative binomial distribution where we can think of $W$ as the number of Bernoulli trials until we get $n$ successes. More precisely, let $\{B_i\}$ for $i\ge1$ be Bernoulli random variables with probability $p$. Then the event $W>t$ has the same probability as $\sum_{i=1}^t B_i < n$. Thus we have that 
    \begin{equation*}
        \P(W > t ) \le \P(\sum_{i=1}^t B_i < n).
    \end{equation*}
    We can now use Chernoff inequality (\Cref{lemma:chernoff}) to get that for $t = 2n/p$:
    \begin{align*}
    \P(\sum_{i=1}^t B_i < n) 
    \le \exp{(-tp/8)} = \exp{(-n/4)}.
    \end{align*}
    This proves that 
    \begin{equation*}
    \P(W > 2n/p ) \le \exp{\left(-n/4\right)}.
    \end{equation*}
    The claim now follows by noticing that $\sum_{i=1}^n W_i \le \sum_{i=1}^k W_i $ for $W_i$ iid geometric random variable when $k \ge n$, thus $\P(\sum_{i=1}^n W_i \ge 2k/p) \le \P(\sum_{i=1}^k W_i \ge 2k/p) \le \exp{\left(-k/4\right)}$

\end{proof}

%% file: bib.bib
@string{colt12 = {Proceedings of the Twenty Fifth Annual Conference on
		  Computational Learning Theory}}

@string{nips2013= {Advances in Neural Information Processing Systems 26}}

@string{nips2021= {Advances in Neural Information Processing Systems 34}}

@string{icml14 = {Proceedings of the 31st International Conference on Machine Learning}}

@string{icml17 = {Proceedings of the 34th International Conference on Machine Learning}}

@string{icml21 = {Proceedings of the 38th International Conference on Machine Learning}}

@string{icml23 = {Proceedings of the 40th International Conference on Machine Learning}}

@string{colt10 = {Proceedings of the Twenty Third Annual Conference on
		  Computational Learning Theory}}

@string{colt18 = {Proceedings of the Thirty First Annual Conference on
		  Computational Learning Theory}}

@string{colt21 = {Proceedings of the Thirty Fourth Annual Conference on
	Computational Learning Theory}}

@string{focs14 = {55th Annual Symposium on Foundations of Computer Science}}

@string{focs17 = {58th Annual Symposium on Foundations of Computer Science}}

@inproceedings{AgarwalSi17,
  title={The price of differential privacy for online learning},
  author={Naman Agarwal and Karan Singh},
  booktitle= icml17,
  pages={32--40},
  year={2017},
}

@inproceedings{AsiFeKoTa21,
author = {Hilal Asi and Vitaly Feldman and Tomer Koren and Kunal Talwar},
title = {Private Stochastic Convex Optimization: Optimal Rates in {$\ell_1$} Geometry},
year = 2021,
booktitle = icml21,
}

@inproceedings{BassilyFeTaTh19,
  title={Private stochastic convex optimization with optimal rates},
  author={Raef Bassily and Vitaly Feldman and Kunal Talwar and Abhradeep Thakurta},
  booktitle= {Advances in Neural Information Processing Systems},
  volume = 32,
  pages={11282--11291},
  year= 2019
}

@inproceedings{BassilySmTh14,
  title={Private empirical risk minimization: {E}fficient algorithms and tight error bounds},
  author={Raef Bassily and Adam Smith and Abhradeep Thakurta},
  booktitle= focs14,
  pages={464--473},
  year= 2014,
}

@misc{Duchi19,
author = {John C. Duchi},
title = {Information Theory and Statistics},
year = 2019,
howpublished = {Lecture Notes for Statistics 311/{EE} 377,
Stanford University},
note = {Accessed May 2019},
url = {http://web.stanford.edu/class/stats311/lecture-notes.pdf},
}

@inproceedings{DworkMcNiSm06,
author = {Cynthia Dwork and Frank McSherry and Kobbi Nissim and Adam Smith},
title = {Calibrating noise to sensitivity in private data analysis},
year = 2006,
booktitle = {Proceedings of the Third Theory of Cryptography Conference},
pages = {265--284},
}

@inproceedings{DworkNaPiRo10,
author = {Cynthia Dwork and Moni Naor and Toniann Pitassi and Guy N Rothblum},
title = {Differential privacy under continual observation},
year = 2010,
booktitle = {Proceedings of the Forty-Second Annual ACM
		  Symposium on the Theory of Computing},
pages = {715--724},
}

@inproceedings{DworkKeMcMiNa06,
author = {Cynthia Dwork and Krishnaram Kenthapadi and Frank McSherry
  and Ilya Mironov and Moni Naor},
title = {Our Data, Ourselves: Privacy Via Distributed Noise Generation},
booktitle = {Advances in Cryptology (EUROCRYPT 2006)},
year = 2006,
}

@article{DworkRo14,
 author = {Dwork, Cynthia and Roth, Aaron},
 title = {The Algorithmic Foundations of Differential Privacy},
 journal = {Foundations and Trends in Theoretical Computer Science},
 volume = {9},
 number = {3 \& 4},
 year = {2014},
 pages = {211--407},
 numpages = {197},
 publisher = {Now Publishers Inc.},
 address = {Hanover, MA, USA},
}

@inproceedings{FeldmanKoTa20,
  title={Private stochastic convex optimization: optimal rates in linear time},
  author={Vitaly Feldman and Tomer Koren and Kunal Talwar},
  booktitle={Proceedings of the 52nd Annual ACM on the Theory of Computing},
  pages={439--449},
  year={2020}
}

@book{MitzenmacherUp05,
	title={Probability and computing: Randomized algorithms and probabilistic 
	analysis},
	author={Mitzenmacher, Michael and Upfal, Eli},
	year={2005},
	publisher={Cambridge University Press}
}

@inproceedings{SmithTh13,
author = {Adam Smith and Abhradeep Thakurta},
title = {({N}early) optimal algorithms for private online
  learning in full-information and bandit settings},
year = 2013,
booktitle = nips2013,
}

@article{SteinkeUl17,
author = {Thomas Steinke and Jonathan Ullman},
title = {Between Pure and Approximate Differential Privacy},
year = 2017,
journal = {Journal of Privacy and Confidentiality},
pages = {3--22},
volume = 7,
number = 2,
}

@inproceedings{SteinkeUll17b,
  title={Tight lower bounds for differentially private selection},
  author={Thomas Steinke and Jonathan Ullman},
  booktitle=focs17,
  pages={552--563},
  year={2017},
  organization={IEEE}
}

@inproceedings{TalwarThZh15,
  title={Nearly optimal private {L}asso},
  author={Kunal Talwar and Abhradeep Thakurta and Li Zhang},
  booktitle={Advances in Neural Information Processing Systems},
  volume={28},
  pages={3025--3033},
  year={2015}
}

@article{KairouzMcSoShThXu21,
  title={Practical and Private (Deep) Learning without Sampling or Shuffling},
  author={Peter Kairouz and Brendan McMahan and Shuang Song and Om Thakkar  and Abhradeep Thakurta and Zheng Xu},
  journal={arXiv:2103.00039 [cs.CR]},
  year={2021}
}

@inproceedings{JainKoTh12,
  title = {Differentially private online learning},
  author = {Prateek Jain and Pravesh Kothari and Abhradeep Thakurta},
  booktitle = colt12,
  year={2012}
}

@inproceedings{JainTh14,
  title={({N}ear) dimension independent risk bounds for differentially private learning},
  author={Prateek Jain and Abhradeep Thakurta},
  booktitle= icml14,
  pages={476--484},
  year={2014}
}

@article{KalaiVe05,
author = {A. Kalai and S. Vempala},
title = {Efficient algorithms for online decision problems},
year = 2005,
journal = {Journal of Computer and System Sciences},
volume = 71,
number = 3,
pages = {291--307},
}

@inproceedings{GeulenVoWi10,
  title = {Regret Minimization for Online Buffering Problems Using the Weighted Majority Algorithm.},
  author = {Sascha Geulen and Berthold V{\"o}cking  and Melanie Winkler},
  booktitle = colt10,
  year = 2010,
}

@inproceedings{ShermanKo21,
  title={Lazy oco: Online convex optimization on a switching budget},
  author={Uri Sherman and Tomer Koren},
  booktitle= colt21,
  year= 2021,
}

@inproceedings{AltschulerTa18,
  title={Online learning over a finite action set with limited switching},
  author={Jason Altschuler and Kunal Talwar},
  booktitle= colt18,
  year= 2018,
}

@inproceedings{AsiDuFaJaTa21,
  title={Private Adaptive Gradient Methods for Convex Optimization},
  author={Hilal Asi and  John Duchi and Alireza Fallah and Omid Javidbakht and Kunal Talwar},
  booktitle = icml21,
  pages={383--392},
  year= 2021,
}

@article{AroraHaKa12,
title = {The multiplicative weights update method: a meta algorithm and
applications},
author = {Sanjeev Arora and Elad Hazan and Satyen Kale},
year = 2012,
journal = {Theory of Computing},
volume = 8,
number = 1,
pages = {121--164},
}

@book{cesa2006prediction,
  title={Prediction, learning, and games},
  author={Cesa-Bianchi, Nicolo and Lugosi, G{\'a}bor},
  year={2006},
  publisher={Cambridge university press}
}

@article{chen2020minimax,
  title={Minimax regret of switching-constrained online convex optimization: No phase transition},
  author={Chen, Lin and Yu, Qian and Lawrence, Hannah and Karbasi, Amin},
  journal={Advances in Neural Information Processing Systems},
  volume={33},
  pages={3477--3486},
  year={2020}
}

@article{JainRaSiSm21,
  title={The Price of Differential Privacy under Continual Observation},
  author={Palak Jain and Sofya Raskhodnikova and Satchit Sivakumar and Adam Smith},
  journal={arXiv:2112.00828 [cs.DS]},
  year= 2021
}

@article{BunUV18,
author = {Bun, Mark and Ullman, Jonathan and Vadhan, Salil},
title = {Fingerprinting Codes and the Price of Approximate Differential Privacy},
journal = {SIAM Journal on Computing},
volume = {47},
number = {5},
pages = {1888-1938},
year = {2018},
doi = {10.1137/15M1033587},

URL = { 
        https://doi.org/10.1137/15M1033587
    
},
eprint = { 
        https://doi.org/10.1137/15M1033587
    
}
,
}

@inproceedings{AsiLeDu21,
 author = {Hilal Asi and Daniel Levy and John Duchi},
 booktitle = nips2021,
 pages = {19069--19081},
 title = {Adapting to function difficulty and growth conditions in private optimization},
 volume = {34},
 year = {2021}
}

@article{AsiFeKoTa22b,
  title={Near-Optimal Algorithms for Private Online Optimization in the Realizable Regime},
  author  = {Hilal Asi and Vitaly Feldman and Tomer Koren and Kunal Talwar},
  journal= icml23 ,
  year= 2023
}
